\documentclass[11pt]{article}

\usepackage[final]{acl}

\usepackage{times}
\usepackage{latexsym}

\usepackage[T1]{fontenc}

\usepackage[utf8]{inputenc}

\usepackage{microtype}

\usepackage{inconsolata}

\usepackage{graphicx}

\usepackage{svg}
\usepackage{color}
\usepackage{tabularx}
\usepackage{tabu}
\usepackage{booktabs}
\usepackage{makecell}
\usepackage{float}

\usepackage{hyperref}
\usepackage{tablefootnote}
\usepackage{xspace}

\usepackage{algorithm}      
\usepackage{algpseudocode}  
\usepackage{amsmath}        
\usepackage{sansmath}
\usepackage{amssymb}        
\usepackage{amsthm}
\usepackage{marvosym}       
\usepackage{xcolor}         
\usepackage{tikz}           
\usepackage{colortbl}
\usepackage{caption}        
\usepackage{tcolorbox}
\usepackage{subcaption}
\usepackage{pdfpages}
\usepackage{arydshln}
\usepackage[normalem]{ulem} 
\usepackage{multirow}
\usepackage{cleveref}

\definecolor{rowgray}{gray}{0.95}
\definecolor{QualcommBlue}{HTML}{3253DC}
\definecolor{ZJUBlue}{RGB}{0,63,136}

\newcommand{\ZJUZ}{\textcolor{ZJUBlue}{\textsf{\textbf{\textit{Z}}}}}
\newcommand{\QualcommQ}{\textcolor{QualcommBlue}{\textsf{\textbf{\textit{Q}}}}}
\newcommand\blfootnote[1]{%
  \begingroup
  \renewcommand\thefootnote{}\footnotetext{#1}%
  \endgroup
}

\usepackage{listings}
\usepackage{enumitem}
\usepackage{amssymb} 
\tcbuselibrary{skins, breakable, listings}

\newcommand{\E}{\mathbb{E}}
\newcommand{\R}{\mathbb{R}}

\newcommand{\ours}{\textsc{Android Coach}\xspace}
\newcommand{\sssa}{\textit{SSSA}\xspace}
\newcommand{\ssma}{\textit{SSMA}\xspace}
\newcommand{\SSSA}{\textit{Single State Single Action}\xspace}
\newcommand{\SSMA}{\textit{Single State Multiple Actions}\xspace}
\newcommand{\up}[1]{\textcolor{red}{(+#1)}}
\newcommand{\down}[1]{\textcolor{blue}{(-#1)}}

\title{\ours: Improve Online Agentic Training Efficiency with\\Single State Multiple Actions}

\author{
  \textbf{Guo Gan\textsuperscript{\ZJUZ*}} \quad
  \textbf{Yuxuan Ding\textsuperscript{\QualcommQ}} \quad
  \textbf{Cong Chen\textsuperscript{\ZJUZ}} \quad
  \textbf{Yuwei Ren\textsuperscript{\QualcommQ}} \quad
  \textbf{Yin Huang\textsuperscript{\QualcommQ}} \quad
  \textbf{Hong Zhou\textsuperscript{\ZJUZ\Letter}}
\\[0.8em]
  \textsuperscript{\ZJUZ}Zhejiang University \quad\quad
  \textsuperscript{\QualcommQ}Qualcomm AI Research
}

\begin{document}
\maketitle
\blfootnote{\textsuperscript{*}This work was done during Guo Gan's internship at Qualcomm AI Research, an initiative of Qualcomm Technologies, Inc. \textsuperscript{\Letter}Corresponding author: Hong Zhou \href{mailto:zhouhong_zju@zju.edu.cn}{<zhouhong\_zju@zju.edu.cn>}. Our code will be available at \url{https://github.com/gguogan/Android_Coach}.}
\begin{abstract}
Online reinforcement learning (RL) serves as an effective method for enhancing Android agents. However, guiding agents to learn through online interaction is prohibitively expensive due to the high latency of emulators and the sample inefficiency of existing RL algorithms. 
We identify a fundamental limitation in current approaches: the \SSSA paradigm, which updates the policy with one-to-one state-action pairs from online one-way rollouts, without fully exploring each costly emulator state.
In this paper, we propose \ours, a novel framework that shifts the training paradigm to \SSMA, allowing the agent to sample and utilize multiple actions for a single online state. 
We enable this without additional emulator overhead by online learning a critic that estimates action values. 
To ensure the critic serves as a reliable coach, we integrate a process reward model and introduce a group-wise advantage estimator based on the averaged critic outputs. 
Extensive experiments demonstrate the effectiveness and efficiency of \ours: it achieves 7.5\% and 8.3\% success rate improvements on AndroidLab and AndroidWorld over UI-TARS-1.5-7B, and attains $1.4\times$ higher training efficiency than \SSSA methods PPO and GRPO at matched success rates.
\end{abstract}

 \begin{figure}[!t]
    \centering
    \includegraphics[width=0.48\textwidth]{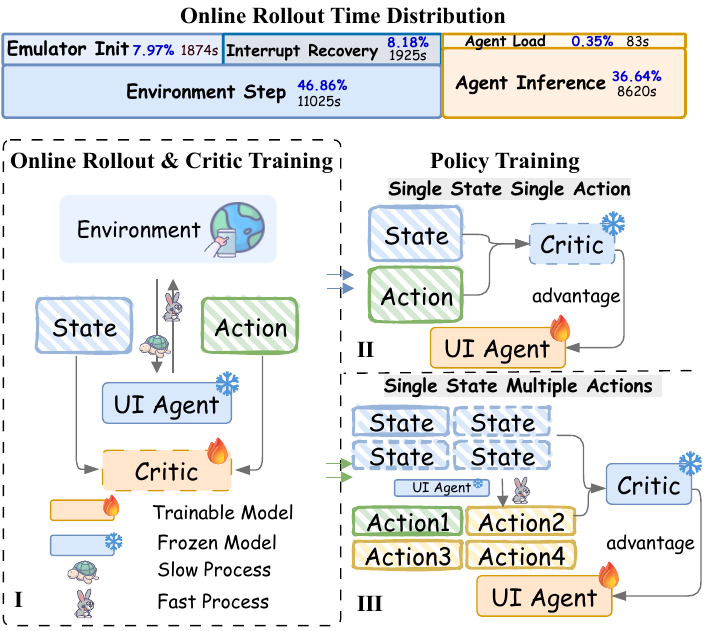}
    \caption{(Top): Online rollout time distribution based on the measured time on 8 parallel environments in training for 80 steps. (Bottom I): The conventional online rollout and critic training loop. The primary bottleneck is the high-latency environmental interaction, while the GUI agent action inference is relatively fast. (Bottom II): Standard agent update with \SSSA paradigm. Agent updates rely merely on the state-action pairs collected from the online rollout. (Bottom III): \ours update with \SSMA paradigm. We fully leverage each expensive online state by generating multiple actions. The agent is then updated using this data. This approach improves training efficiency by gathering more training samples within the same online interaction cost.}
    \label{fig:ssma}
\end{figure}
\section{Introduction}
Graphical User Interface (GUI) agent is an application of Vision-Language models (VLMs) in interactive scenarios~\cite{appagent,autoui, yang2025qwen3}. 
When human users provide an instruction, the agent leverages reasoning and function-calling capabilities to autonomously conduct multi-turn interactions to complete the task~\cite{xu2026odysseyarena, lian2025ui, xie-etal-2025-gui, chen2026surveyinductivereasoninglarge}. Reinforcement learning (RL) is widely used in agent training, which helps to enhance reasoning and decision-making capability for complex sequential tasks~\cite{lu2025arpo}. In this paper, we focus on optimizing the reinforcement learning for GUI agent.

Reinforcement learning approaches for GUI agent generally fall into two categories based on their interaction paradigm.
Offline approaches rely on pre-collected expert trajectories~\cite{sun2025genesis, wei2026anchor}. While they avoid frequent environment interactions, the methods are bounded by the data quality and struggle to handle rapid application and GUI updates due to the lack of online exploration~\cite{bai2024digirl, lu2025uir1}.
Online methods mitigate these limitations by collecting data through environment interactions and learning from active trial-and-error for better performance~\cite{bai2025digi}, but still exhibit other shortcomings in training efficiency.

Online training typically suffers from poor sample efficiency~\cite{lu2025uis1, dong2025agenticreinforcedpolicyoptimization} as shown in Figure~\ref{fig:ssma}. First, it requires high-latency emulator interactions, including initialization, recovery and reaction, which is $1.7\times$ greater than the time for model loading and inference. Consequently, the online RL states, which include screenshots and interaction history, are costly to collect.
Second, current online RL methods conduct one-to-one state-action rollouts~\cite{ zhang-etal-2025-agentcpm}. This means the agent can only sample once with a given state, because the emulator would transition to the next state after the execution. We term this the \SSSA (\sssa) paradigm, like PPO in UI-TARS~\cite{wang2025ui} and GRPO in ARPO~\cite{lu2025arpo}. This paradigm makes it difficult to sufficiently explore the state, since the agent cannot try other actions.

We propose \ours, a novel actor-critic framework that adopts \SSMA (\ssma) paradigm to address the problems above: 1) To reduce interaction overhead, we use a critic (\emph{i.e.}, state-action value function Q) to estimate action value, which allows us to get values of more sampled actions without the environment. 2) To sufficiently explore the states, we randomly sample multiple actions given online states and value them with Q, which means the agent can do more exploration without additional emulator overhead.

Reliably evaluating the action value is essential in our paradigm. Our Q is kept updated using the actor online rollout data, which ensures the robustness against the distribution shift typically encountered in offline approaches~\cite{zheng2025vem}. Meanwhile, we introduce a fine-grained, pretrained process reward model into our framework rather than merely trajectory-level outcome supervision. This makes Q capable of crediting correct steps within a failed trajectory, leading to a better supervision of intermediate steps~\cite{chen2025gui}. Besides, tailored to \ssma paradigm, we propose a novel advantage estimation method Actor-Critic Leave-one-out(ACLOO), where the baseline is the average Q-value with leave-one-out strategy. Our design reduces estimation variance without effort to train a state-value model, and introduces the relative quality, guiding the agent based on the average level~\cite{konda1999actor,bai2024digirl}. This is inspired by RLOO~\cite{kool2019buy}, where the advantage is reward-based while ours uses long-term value.

 \begin{figure*}[ht]
    \centering
    \includegraphics[width=\textwidth]{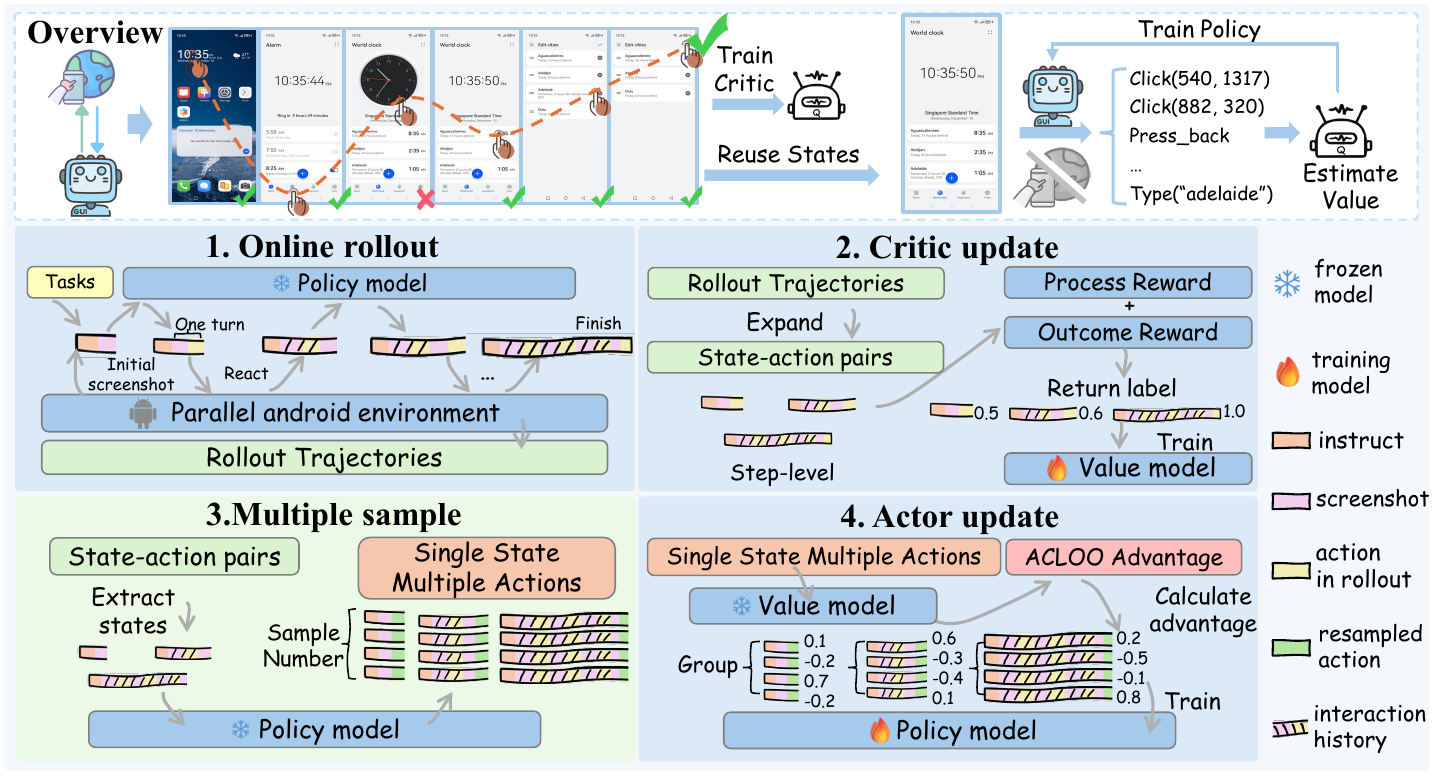}
    \vspace{-2em}
    \caption{Overview and pipeline for a training step in \ours. 1. Online Rollout: Policy interacts with parallel environments to collect complete trajectories. 2. Critic Update: Annotate state-action pairs with returns to train the value model. 3. Multiple sample: Resample multiple actions for each online state. 4. Actor Update: Compute action values and advantages with leave-one-out strategy, then update policy via gradient step.}
    \label{fig:pipeline}
\end{figure*}
The overall framework of \ours is shown in Figure~\ref{fig:pipeline}. This is an online actor-critic method with \SSMA to increase the number of training samples, making full use of the online rollout state within the same interactions. The actor samples multiple actions and constructs state-action pairs as training samples. To do \ssma training with fewer interactions, \ours does not execute these actions with the emulator, but evaluates them with the critic, where a leave-one-out advantage helps train the actor. For critic training, \ours applies the return of online rollout actions as ground truth which is estimated by integrating the process reward and outcome reward.
We validate our approach on the AndroidLab~\cite{xu2024androidlab} and AndroidWorld~\cite{rawles2024androidworld} benchmarks, achieving a 7.5\% and 8.3\% improvement over the success rate of original UI-TARS~\cite{qin2025ui}, while outperforming conventional \SSSA methods in GUI agent RL including PPO~\cite{schulman2017proximal} and GRPO~\cite{shao2024deepseekmath} with 1.4x training efficiency.

In summary, our contributions are as follows:
\begin{enumerate}
\item We propose \ours, a meticulously designed framework that first enables \SSMA paradigm for efficient online agentic reinforcement learning to the best of our knowledge.
\item We propose an online-trained critic guided by both outcome and process rewards, together with our leave-one-out advantage estimator. Without additional environment overhead, the critic supports \SSMA paradigm with reliable action advantage.
\item Extensive experiments on dynamic benchmarks demonstrate the training efficiency of \ours and the effectiveness of its components in online reinforcement learning.
\end{enumerate}

\section{Related Work}
\subsection{RL for Training GUI Agents}
Existing RL methods for GUI agents generally fall into offline and online categories as shown in Table~\ref{tab:related}. Offline approaches~\cite{sun2025genesis, lu2025uis1, luo2026navimasterlearningunifiedpolicy} rely on extensive expert or pre-collected datasets. Consequently, such methods are limited by data quality and fail to handle environment updates effectively~\cite{intelligence2025pi06vlalearnsexperience}. In contrast, online training enables continuous improvement through exploration and trial-and-error in real environments~\cite{shi2025mobileguirladvancingmobilegui, ye2025mobile}. Hybrid frameworks like DigiRL~\cite{bai2024digirl} integrate both phases, while online methods like MobileRL~\cite{xu2025mobilerl}, GUI-Shepherd~\cite{chen2025gui}, and UI-TARS-2~\cite{wang2025ui} adopt GRPO~\cite{shao2024deepseekmath} or PPO~\cite{schulman2017proximal} for online optimization. However, interacting with Android emulator is time-consuming. Although environment parallelization and replay buffers~\cite{lu2025arpo} partially alleviate this issue, current online methods remain restricted to \SSSA paradigm. These approaches generate only one action per state, which fails to fully exploit the expensive state and necessitates more interaction steps for better performance. In this paper, we introduce \SSMA paradigm for online RL which generates and evaluates multiple actions for each state, significantly enhancing training efficiency under limited interaction budgets.
\begin{table*}[htbp]
\centering
\label{tab:method_compare}
\resizebox{1\textwidth}{!}{%
\begin{tabular}{lllccc}
\toprule
\textbf{Training Mode} & \textbf{Base Algorithm} & \textbf{Method} & \textbf{Advantage Estimation} & \textbf{Exploration Paradigm} & \textbf{Sample Efficiency} \\ \midrule
\multirow[t]{4}{*}{Offline} & SFT & OS-Genesis~\cite{sun2025genesis} & Reward-based & Static Data & N/A \\
 & Archer~\cite{zhou2024archer} & DigiQ~\cite{bai2025digi} & Value-based & Static Data & N/A \\
 & Q-Learning~\cite{watkins1989learning} & VEM~\cite{zheng2025vem} & Value-based & Static Data & N/A \\ 
& DPO~\cite{rafailov2023direct} & UI-TARS-1.5~\cite{qin2025ui} & Pairwise-Preference & Static Data & N/A \\
\noalign{\vskip 0.5ex}\hdashline\noalign{\vskip 0.5ex}
Hybrid & AWR~\cite{peng2019advantage} & DigiRL~\cite{bai2025digi} & Value-based & SSSA & Low \\ 
\noalign{\vskip 0.5ex}\hdashline\noalign{\vskip 0.5ex}
\multirow[t]{2}{*}{Online} & GRPO~\cite{shao2024deepseekmath} & \begin{tabular}[t]{@{}l@{}} MobileRL~\cite{xu2025mobilerl}\\WebAgent-R1~\cite{wei2025webagent} \\ ARPO~\cite{lu2025arpo} \end{tabular} & Reward-based & SSSA & Low \\
\noalign{\vskip 0.5ex}
 & PPO~\cite{schulman2017proximal} & \begin{tabular}[t]{@{}l@{}} UI-TARS-2~\cite{wang2025ui} \\ GUI-Shepherd~\cite{chen2025gui} \end{tabular} & Value-based & SSSA & Low \\ \midrule
\textbf{Online} & ACLOO & \textbf{\ours (Ours)} & \textbf{Value-based} & \textbf{SSMA} & \textbf{High} \\ \bottomrule
\end{tabular}%
}
\caption{Comparison of representative GUI agent training frameworks. \ours is the first to achieve efficient \SSMA (SSMA) exploration in an online setting, significantly improve sample efficiency with limited interaction costs compared to standard \SSSA (SSSA) approaches. ACLOO refers to Actor-Critic Leave-One-Out we proposed in our method.}
\label{tab:related}
\vspace{-10pt}
\end{table*}
\subsection{Advantage Estimation in RL Training}
Advantage estimation is a critical component for policy optimization in modern GUI agents. Trajectory-level reward-based approaches~\cite{xu2025mobilerl, lu2025arpo, luo2025gui, wanyan2025look, zhang2026dontactblindlyrobust} use GRPO, requiring full rollouts to obtain outcome rewards and compute advantages from averaged returns.
Critic-based methods~\cite{chen2025gui, wang2025ui, bai2024digirl} can estimate action values without full rollouts, but in online settings the critic is typically trained only with outcome rewards, as process rewards are hard to obtain during interaction. In contrast, offline methods like VEM~\cite{zheng2025vem} utilize step-level supervision, but inherit the limitations of offline training. In this paper, we design a critic that incorporates a process reward mechanism for GUI tasks. Furthermore, actor-critic methods using Q-functions~\cite{bai2025digi, zhou2024archer} usually introduce an additional state-value model to reduce variance and stabilize training. To avoid this extra component, we propose an average Q-value baseline with leave-one-out strategy for advantage estimation.

\subsection{RL with Multiple Actions Estimation}
Several established RL methods share the core principle of evaluating multiple actions per state for sample efficiency or training stability. Discrete SAC~\cite{christodoulou2019soft} computes the exact expected state value by iterating over all possible actions, yielding zero-variance gradient updates. Similarly, Expected Policy Gradients (EPG)~\cite{ciosek2018expected} eliminates sampling variance entirely via analytic integration for continuous distributions or exhaustive evaluation for discrete ones. However, exact marginalization rather than probabilistic sampling with variance reduction tricks strictly limits these methods to small action spaces (e.g., Atari). Consequently, they are computationally intractable for reasoning VLM-based GUI agents with combinatorially massive action spaces. GRPO~\cite{shao2024deepseekmath} is a prevalent multiple-rollout method. However, dynamic and irreversible GUI environments prohibit parallel follow-up executions from the exact same state, causing step-level GRPO to fail to acquire long-term value supervision in the absence of value model, while sequence-level GRPO is still limited to Single State Single Action. Given these constraints of dynamic GUI tasks, our method adopts Monte Carlo sampling. To fulfill our core objective of improving online training efficiency via Single State Multiple Actions paradigm without extra interaction overhead, we introduce a learned value model to estimate sample advantages with our group-wise ACLOO baseline design.
\section{Android Coach}
In this section, we first present the \textbf{preliminary} knowledge including problem formulation and actor-critic framework (Section~\ref{sec:preliminary}). Then we introduce \textbf{how to train a reliable critic} (state-action value function Q) for accurate action value estimation with process reward model and online updates (Section~\ref{sec:coach_train}). Finally, we present \textbf{how to leverage the critic} to improve sample efficiency with \SSMA paradigm and our proposed Actor-Critic Leave-One-Out advantage estimation method (Section~\ref{sec:rl_train}).

\subsection{Preliminary}
\label{sec:preliminary}
\paragraph{Problem Formulation.}
We formulate the Android agent task as a finite-horizon Markov Decision Process (MDP) defined by the tuple $(\mathcal{S}, \mathcal{A}, \mathcal{R})$.
Given an instruction $I$, the process begins at state $s_0$ with the initial GUI screenshot.
At each timestep $t$, the state $s_t \in \mathcal{S}$ consists of the instruction $I$ and the interaction history. 
The policy $\pi_\theta(\cdot|s_t)$ parameterized by $\theta$ samples an $a_t$ from action space $\mathcal{A}$, which is constructed by reasoning-based plan and an operation from the GUI operation action space in Appendix~\ref{app:exp_detail}.
After execution, the environment transitions to the next state $s_{t+1}$.
This interaction loop continues until task completion or a maximum step limit is reached.
The agent receives binary rewards $r_t \in \{0, 1\}\sim \mathcal{R}$, including \textit{process reward} for step-wise correctness and \textit{outcome reward} for final success.
The objective is to learn the optimal parameters $\theta$ that maximize the expected cumulative discounted return:
$J(\pi_\theta) = \mathbb{E}_{\pi_\theta} \left[ \sum_{t=0}^{T} \gamma^t r_t \right],$
where $\gamma$ is the discount factor.

\paragraph{Actor-Critic Framework.}
In the online Actor-Critic method with Q function, the policy is learned concurrently with a state-action value function $Q_\phi(s_t, a_t)$, parameterized by $\phi$. The Critic learns to estimate the expected cumulative reward after taking action $a_t$ in state $s_t$ and following the policy $\pi$ thereafter. This learned Critic $Q_\phi$ then provides an evaluative signal in the form of an Advantage $A(s_t, a_t)$ to guide the update of Actor policy $\pi_\theta$.

\subsection{Train a Reliable Coach}
\label{sec:coach_train}
We build a reliable state-action value function Q by 1) training a process reward model and 2) online training of the Q incorporating process reward.
\paragraph{Process Reward Model (PRM).}
We first design a PRM to provide step-wise process rewards for intermediate actions. Motivated by the effectiveness of model reasoning~\cite{wei2022chain, guo2025deepseek, wanyan2025look, qu2026matchtir, chen2025perturbollava}, we train PRM as a reasoning model that generates analysis prior to judgment. As shown in Figure~\ref{fig:prm_pipeline}, we construct the training dataset by aggregating trajectories from offline datasets AndroidControl~\cite{li2024effects} and supplementary preliminary rollout data, retaining only success and non-redundant trajectories.
For each step $(I, s_t, a_t)$, the initial policy conducts single-step generation to produce a reasoning context alongside a predicted action. Each sample is labeled positive if the generated action matches the ground-truth one from the original success trajectory, and negative otherwise.
This process yields 20k data points, each structured as $(I, s_t, a_t, \textit{label})$. Then we inject the reason for the process reward label with GPT-4o~\cite{hurst2024gpt}. More details are provided in Appendix~\ref{sec:app_bench_data}.
\begin{figure}[t!]
    \centering
    \includegraphics[width=0.5\textwidth]{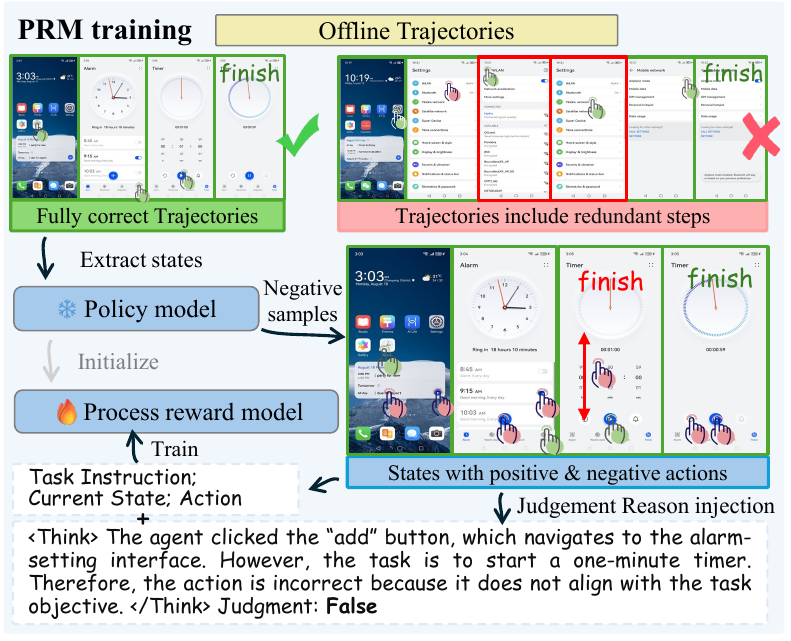}
    \caption{An overview of dataset construction pipeline for process reward model training.}
    \label{fig:prm_pipeline}
\end{figure}

We initialize PRM with our initial policy model and perform full-parameter supervised fine-tuning (SFT) with prompt in Figure~\ref{fig:prompt_prm} in Appendix. The PRM parameterized by $\beta$ is updated to optimize the cross-entropy (CE) loss which denoted as $\mathcal{L}_{\text{PRM}}(\beta)$ in Equation~\ref{eq:sft_loss}:\begin{equation}
    \mathcal{L}_{\text{PRM}}(\beta) = - \log P_{\beta}(y | I, s_t, a_t)
    \label{eq:sft_loss}
\end{equation}
\paragraph{Online Critic Training.}
The critic Q-value function ${Q}_{\phi}$ shares the same architecture as the policy model, augmented with a value head~\cite{vonwerra2022trl}, serving as a coach.
We start the RL loop as shown in first two stages of Figure~\ref{fig:pipeline}, in which the policy model interacts with multiple parallel environments at first to collect a batch of trajectories, $\mathcal{D}$. Upon completion of the batch rollout, the pre-trained PRM assigns step-level process rewards $r_p^t$ based on the intermediate actions, and the outcome verifier (OV) assigns outcome reward ($r_{\text{o}}$) based on the final result.
We estimate the target return $R_t$ for each state-action pair using a weighted Monte Carlo estimation that combines these rewards with weight parameters $\omega_p$, $\omega_o$ and discount factor $\gamma$: $R_t = \omega_p \sum_{\tau=t}^T \gamma^{T-\tau} r_{p}^{\tau:T} + \omega_o r_{o}$. The critic $Q_\phi$ is subsequently updated by minimizing the clipped mean squared error loss between its predictions $Q_\phi(s_t, a_t)$ and the estimated target returns $R_t$ as shown in Equation~\ref{eq:critic_mse}:\begin{equation}
\resizebox{0.89\linewidth}{!}{$
\mathcal{L}_{\text{Q}}(\phi) = \frac{1}{2} \mathbb{E}_{t} \left[ \max\left( (Q_{\phi} - R_t)^2, (\text{clip}(Q_{\phi}, Q_{{old}} \pm \epsilon_{\text{v}}) - R_t)^2 \right) \right]
$}
\label{eq:critic_mse}
\end{equation}
However, directly training critic together with actor presents a significant challenge in providing reliable value, which is also noted by DigiQ~\cite{bai2025digi}. We contend that this issue arises because the value model is poorly prepared for GUI tasks value estimation at the beginning of training, resulting in misleading guidance for the policy updates~\cite{bai2025digi, wang2025ui}. Consequently, before online RL, we initialize the model by pre-training it with the PRM dataset, where labels are mapped to binary scores.
\begin{table*}[htbp]
\centering
\resizebox{\textwidth}{!}{%
    \begin{tabular}{lcccccccc}
    \toprule
    \multirow{2}{*}{\textbf{Models}} & \multirow{2}{*}{\textbf{\#Params}} & \multicolumn{3}{c}{\textbf{AndroidLab SR (\%)}} & \multicolumn{4}{c}{\textbf{AndroidWorld SR (\%)}} \\
    \cmidrule(lr){3-5} \cmidrule(lr){6-9}
     & & \textbf{QD} & \textbf{OP} & \textbf{Average} & \textbf{Easy} & \textbf{Mid} & \textbf{Hard} & \textbf{Average} \\
    \midrule
    
    \rowcolor{rowgray}\multicolumn{9}{l}{\textit{Proprietary Models}} \\
    Gemini-Pro-1.5 (SoM)~\cite{team2024gemini} & - & - & - & 16.7 & - & - & - & 22.8 \\
    GPT-4o (SoM)~\cite{hurst2024gpt} & - & - & - & 31.2 & - & - & - & 34.5 \\
    Claude-Sonnet-4 (SoM)~\cite{claude4} & - & - & - & 40.6 & - & - & - & \textbf{41.0} \\
    UI-Genie-Agent~\cite{xiao2025ui}  & 72B & - & - & \textbf{41.2} & - & - & - & - \\
    \midrule
    \rowcolor{rowgray}\multicolumn{9}{l}{\textit{Open-source 32B/72B Models}} \\
    Qwen2.5VL-32B-Instruct~\cite{bai2025qwen2} & 32B & 28.4\textcolor{gray}{\scriptsize\,±2.1} & 25.8\textcolor{gray}{\scriptsize\,±1.1} & 28.5\textcolor{gray}{\scriptsize\,±0.4} & 37.2\textcolor{gray}{\scriptsize\,±1.9} & 14.8\textcolor{gray}{\scriptsize\,±3.2} & 10.5\textcolor{gray}{\scriptsize\,±5.3} & 25.9\textcolor{gray}{\scriptsize\,±0.9} \\
    UI-TARS-72B-DPO~\cite{qin2025ui}  & 72B & 36.4\textcolor{gray}{\scriptsize\,±2.8} & \underline{31.5}\textcolor{gray}{\scriptsize\,±1.6} & \underline{35.5}\textcolor{gray}{\scriptsize\,±2.2} & \textbf{57.4}\textcolor{gray}{\scriptsize\,±1.6} & 27.8\textcolor{gray}{\scriptsize\,±5.6} & 10.5\textcolor{gray}{\scriptsize\,±0.0} & \underline{40.5}\textcolor{gray}{\scriptsize\,±0.9} \\
    \midrule
    
    \rowcolor{rowgray}\multicolumn{9}{l}{\textit{Open-source 7B/8B Models}} \\
    OS-Genesis-7B-AW~\cite{sun2025genesis} & 7B & 6.8\textcolor{gray}{\scriptsize\,±2.8} & 3.6\textcolor{gray}{\scriptsize\,±1.2} & 5.1\textcolor{gray}{\scriptsize\,±1.9} & 26.8\textcolor{gray}{\scriptsize\,±2.5} & 11.1\textcolor{gray}{\scriptsize\,±0.0} & 1.8\textcolor{gray}{\scriptsize\,±3.0} & 17.8\textcolor{gray}{\scriptsize\,±1.3} \\
    Qwen2.5-VL-7B-Instruct~\cite{bai2025qwen2}  & 7B & 14.8\textcolor{gray}{\scriptsize\,±1.9} & 4.3\textcolor{gray}{\scriptsize\,±0.0} & 8.9\textcolor{gray}{\scriptsize\,±0.7} & 23.5\textcolor{gray}{\scriptsize\,±2.5} & 6.5\textcolor{gray}{\scriptsize\,±4.2} & 3.5\textcolor{gray}{\scriptsize\,±3.0} & 14.9\textcolor{gray}{\scriptsize\,±0.5} \\
    AgentCPM-GUI-8B~\cite{zhang-etal-2025-agentcpm} & 8B & 8.6\textcolor{gray}{\scriptsize\,±1.1} & 16.8\textcolor{gray}{\scriptsize\,±0.6} & 14.7\textcolor{gray}{\scriptsize\,±0.4} & 29.0\textcolor{gray}{\scriptsize\,±0.9} & 5.6\textcolor{gray}{\scriptsize\,±2.8} & 3.5\textcolor{gray}{\scriptsize\,±3.0} & 17.5\textcolor{gray}{\scriptsize\,±0.5} \\
    \midrule
    
    \rowcolor{rowgray}\multicolumn{9}{l}{\textit{UI-TARS-1.5-7B Model}~\cite{qin2025ui}} \\
    Base Model  & 7B & 34.0\textcolor{gray}{\scriptsize\,±2.8} & 27.6\textcolor{gray}{\scriptsize\,±1.2} & 31.9\textcolor{gray}{\scriptsize\,±0.7} & 43.7\textcolor{gray}{\scriptsize\,±4.1} & 25.9\textcolor{gray}{\scriptsize\,±1.6} & 10.5\textcolor{gray}{\scriptsize\,±0.0} & 32.8\textcolor{gray}{\scriptsize\,±2.3} \\
    \quad w/ GRPO~\cite{shao2024deepseekmath} & 7B & 36.4\textcolor{gray}{\scriptsize\,±1.1} & 30.5\textcolor{gray}{\scriptsize\,±2.7} & 34.8\textcolor{gray}{\scriptsize\,±1.4} & 51.9\textcolor{gray}{\scriptsize\,±4.1} & \textbf{28.7}\textcolor{gray}{\scriptsize\,±3.2} & 12.3\textcolor{gray}{\scriptsize\,±3.0} & 38.2\textcolor{gray}{\scriptsize\,±2.2} \\
    \quad w/ PPO~\cite{schulman2017proximal} & 7B & \underline{38.9}\textcolor{gray}{\scriptsize\,±3.7} & 29.4\textcolor{gray}{\scriptsize\,±0.6} & 35.0\textcolor{gray}{\scriptsize\,±1.1} & 50.8\textcolor{gray}{\scriptsize\,±1.6} & 26.9\textcolor{gray}{\scriptsize\,±1.6} & \underline{14.0}\textcolor{gray}{\scriptsize\,±3.0} & 37.4\textcolor{gray}{\scriptsize\,±1.8} \\
    \quad w/ Android Coach & 7B & \textbf{42.6}\textcolor{gray}{\scriptsize\,±1.9} & \textbf{33.7}\textcolor{gray}{\scriptsize\,±0.6} & \textbf{39.4}\textcolor{gray}{\scriptsize\,±0.8} & \underline{56.3}\textcolor{gray}{\scriptsize\,±2.5} & \underline{27.8}\textcolor{gray}{\scriptsize\,±2.8} & \textbf{17.5}\textcolor{gray}{\scriptsize\,±3.0} & \textbf{41.1}\textcolor{gray}{\scriptsize\,±1.8} \\

    \bottomrule
    \end{tabular}%
}
\caption{Success rates of proprietary and open-source models on AndroidWorld and AndroidLab for mobile GUI interaction tasks. QD is the abbreviation for the query detect type, and OP is the abbreviation for the operation type. Standard deviations are reported in gray subscripts for all models except proprietary ones.}
\label{tab:main-results}
\end{table*}

\subsection{RL Guided by Coach}
\label{sec:rl_train}
Here we introduce our key designs in actor training: 1) sampling multiple actions and 2) update actor with the Actor-Critic Leave-One-Out advantage.
\paragraph{Multiple Actions from the Online State.}
As shown in third stage of Figure~\ref{fig:pipeline}, by employing the Q function, we can naturally enhance sample efficiency through the \SSMA paradigm without additional interactions.
Specifically, we reuse every costly state $s_t$ from the online trajectories $\mathcal{D}$, which is collected by the policy $\pi_{\theta}$ during the online rollout in one training step. We sample a set of $k$ candidate actions $\{a_t^1, ..., a_t^k\}$ for every $s_t$ using the current policy.

\paragraph{Actor-Critic Leave-One-Out.}
As shown in the last stage of Figure~\ref{fig:pipeline}, to mitigate the high variance associated with policy gradient updates with $Q_\phi$, it is standard practice to subtract a baseline when estimating the advantage $A(s_t, a_t)$.
Conventional Actor-Critic methods typically learn a separate state-value function $V_\psi(s_t)$ to serve as this baseline, \emph{i.e.}, $A(s_t, a_t) = Q_\phi(s_t, a_t) - V_\psi(s_t)$.
However, within our \ssma framework where multiple actions are evaluated for each state, a more direct and potentially more effective baseline is available.
Inspired by Reinforce Leave-One-Out (RLOO)~\cite{ahmadian2024back, kool2019buy}, we propose the Actor-Critic Leave-One-Out (ACLOO) advantage estimation method.
Specifically, given the set of $k$ actions $\{a_t^1, \dots, a_t^k\}$ sampled \emph{i.i.d.} from the current policy $\pi_\theta(\cdot|s_t)$, we define the ACLOO advantage estimate for action $a_t^i$ as:\begin{equation}
\resizebox{0.80\linewidth}{!}{$
\hat{A}_t^i = Q_\phi(s_t, a_t^i) - \frac{1}{k-1} \sum_{j \neq i} Q_\phi(s_t, a_t^j)
$}
\label{eq:acloo}
\end{equation}
This ACLOO advantage estimation offers two key benefits within our framework: (1) It eliminates the need for a separate value network while effectively reducing variance without bias (proven in Appendix~\ref{sec:lemma}); (2) This mechanism inherently captures the \textit{relative quality} among the candidate actions at that state, steering the policy towards learning actions that outperform the average. Then the actions and advantages are used in performing gradient updates on the actor policy $\pi_\theta$ with PPO clip surrogate loss:
\begin{equation}
\resizebox{0.89\linewidth}{!}{$
\mathcal{L}_{policy}(\theta) = \mathbb{E}_t \big[ \min\big( \tfrac{\pi_\theta(a_t|s_t)}{\pi_{\theta_{\text{old}}}(a_t|s_t)} A_t, \text{clip}(\tfrac{\pi_\theta(a_t|s_t)}{\pi_{\theta_{\text{old}}}(a_t|s_t)}, 1{\pm}\epsilon) A_t \big) \big]$}
\label{eq:ppo_clip_aligned}
\end{equation}

\subsection{Putting it Together}
A pseudocode of our algorithm is provided in Appendix~\ref{sec:algorithm}. 
Initially, the policy model performs online rollouts in parallel emulators to collect a batch of trajectories. The critic is updated using returns provided by the Process Reward Model and Outcome Verifier with Equation~\ref{eq:critic_mse}. Subsequently, we update the actor by resampling multiple actions for each online state, computing Q-values and advantages with the updated critic, and applying the policy gradient with Equation~\ref{eq:ppo_clip_aligned}.
\begin{figure*}[t!]
    \centering
    \begin{subfigure}[b]{0.32\textwidth}
        \centering
        \includegraphics[width=\textwidth]{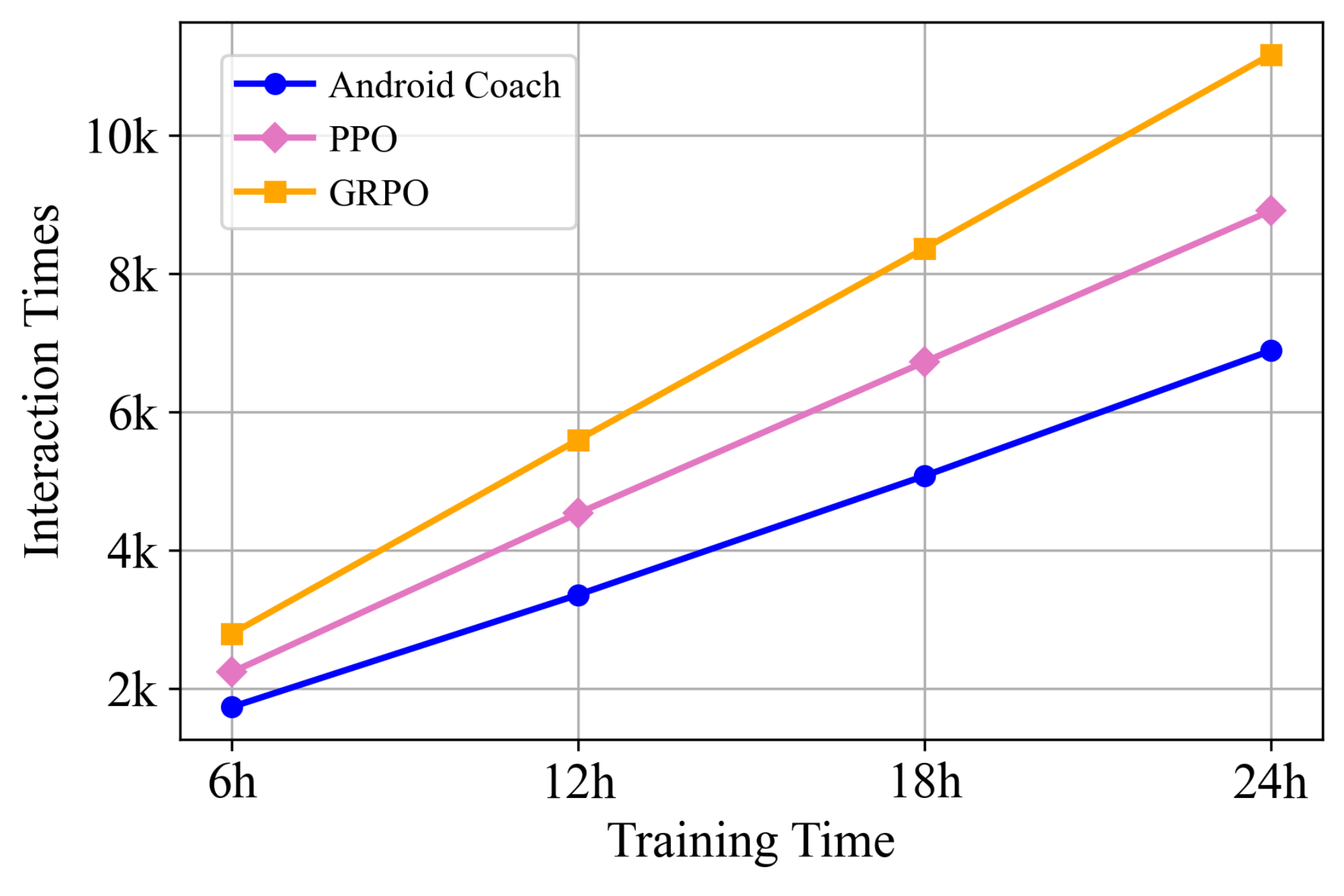}
        \caption{Interaction Times}
        \label{fig:eff_inter}
    \end{subfigure}
    \hfill 
    \begin{subfigure}[b]{0.32\textwidth}
        \centering
        \includegraphics[width=\textwidth]{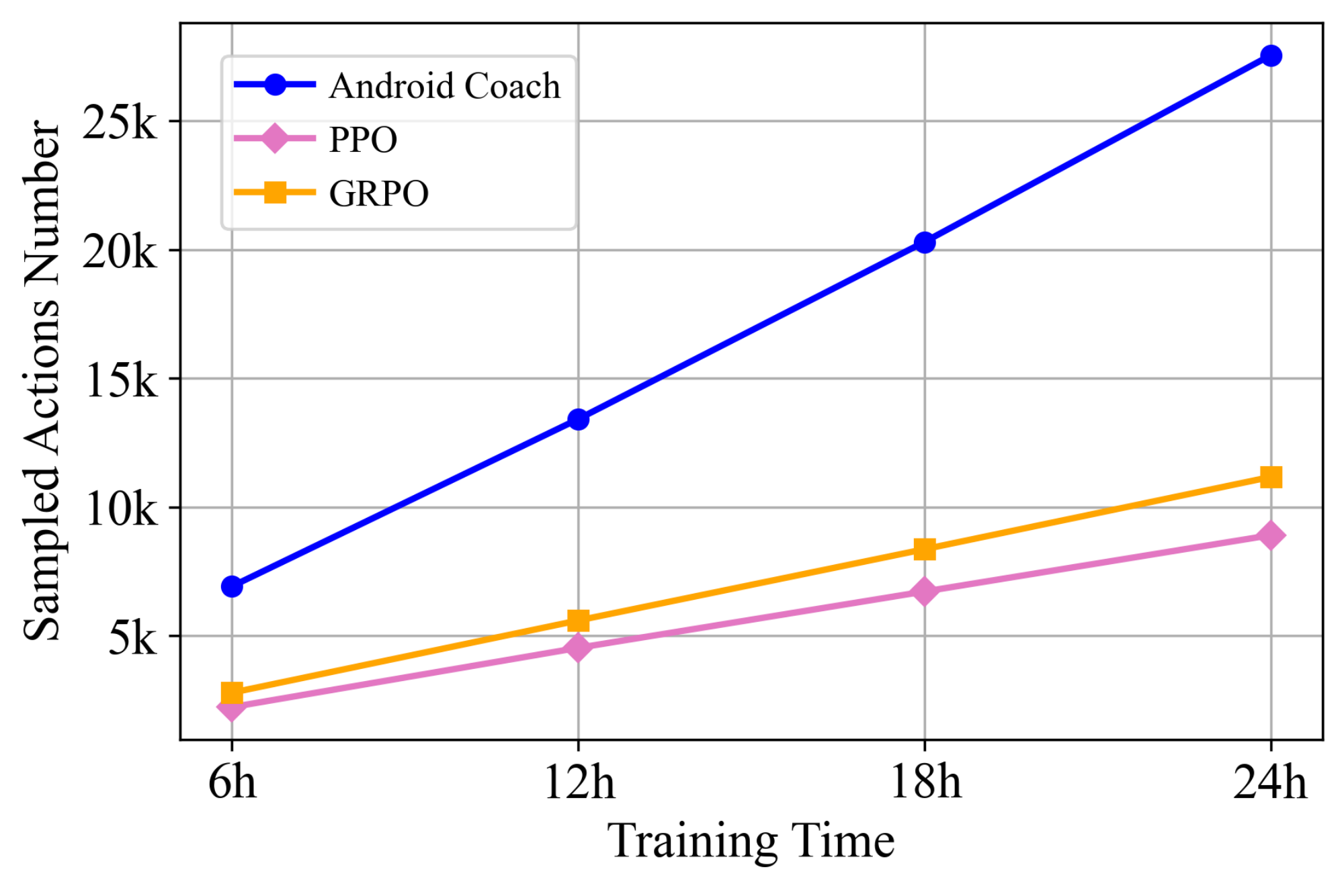}
        \caption{Sampled Actions Number}
        \label{fig:eff_sample}
    \end{subfigure}
    \hfill 
    \begin{subfigure}[b]{0.32\textwidth}
        \centering
        \includegraphics[width=\textwidth]{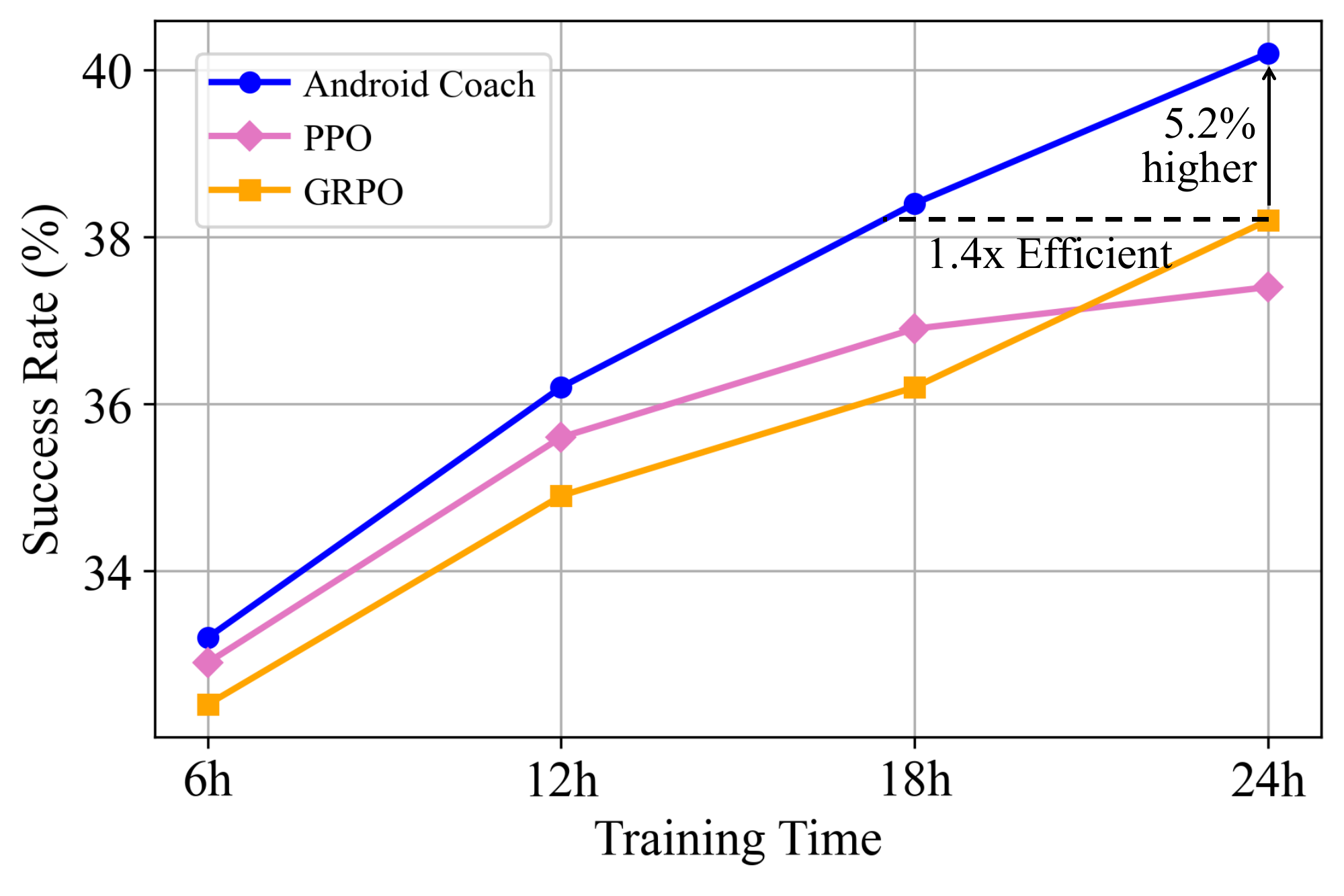}
        \caption{Success Rate}
        \label{fig:eff_sr}
    \end{subfigure}
    
    \caption{Training efficiency analysis of different methods over training time. We report the relationship between training time and (a) interaction times, (b) sampled actions number, and (c) success rate. The GRPO group size and sample number of \ours are both 4. Data is collected on AndroidWorld over UI-TARS-1.5-7B.}
    \label{fig:efficiency}
\end{figure*}

\section{Experiments}

\begin{figure}[!t]
    \centering
    
    \begin{subfigure}{0.46\textwidth} 
        \centering
        \includegraphics[width=\linewidth]{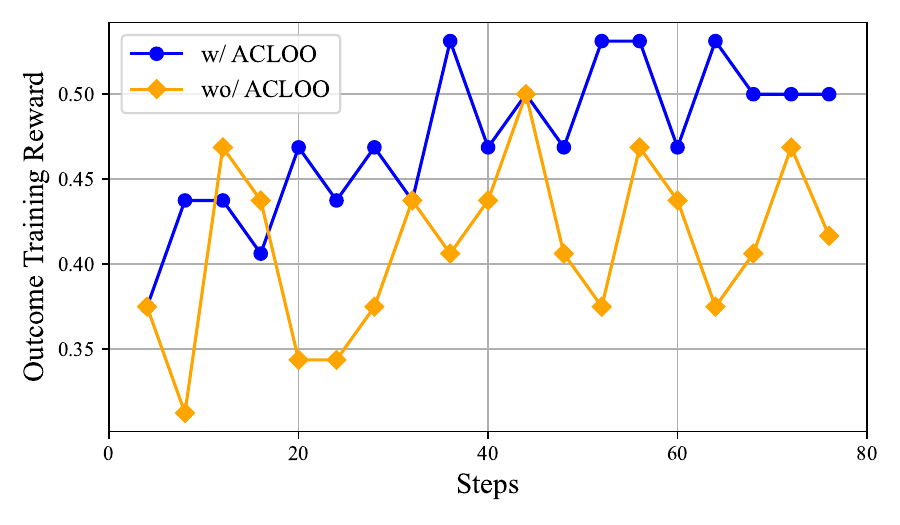}
        \caption{Training outcome reward (4-steps average).} 
        \label{fig:reward_acloo} 
    \end{subfigure}
    
    
    \begin{subfigure}{0.5\textwidth}
        \centering
        \includegraphics[width=\linewidth]{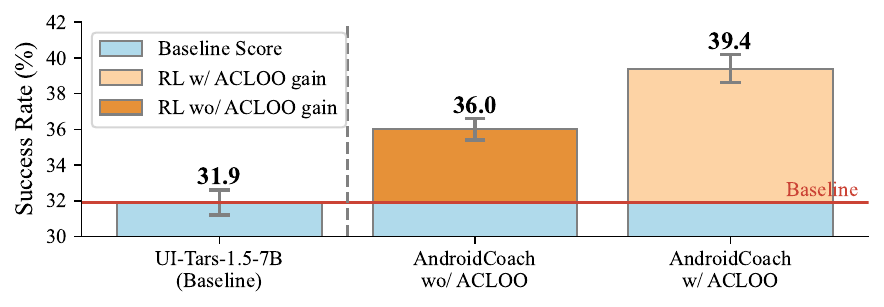}
        \caption{Success rate on AndroidLab.}
        \label{fig:bar_acloo} 
    \end{subfigure}
    
    \caption{The effect of the ACLOO advantage estimation. (a) The training reward curves. (b) The final success rate gain on AndroidLab.} 
    \label{fig:filter_exp}
\end{figure}
\subsection{Experiment Setup}
\paragraph{Environment and Benchmarks.}
We train the agents in parallel Android emulators running Android~13, where the agent interacts via \textit{uiautomator2}. Following prior work~\cite{xu2025mobilerl,xiao2025ui,chen2025gui}, we evaluate on the AndroidLab~\cite{xu2024androidlab} and AndroidWorld~\cite{rawles2024androidworld} benchmarks. AndroidLab contains 138 tasks covering both query detection and operation execution, while AndroidWorld includes 116 tasks with easy/medium/hard difficulties and randomized parameters for diverse scenarios. Task success rate (SR) is computed using the built-in rule-based rubrics. Further details are provided in Appendix~\ref{sec:app_bench_data}.

\paragraph{Dataset and RL Outcome Verifier.}We construct the training dataset by combining randomized tasks from AndroidWorld with the self-collected AndroidLab tasks to form a collection of 2k tasks. Training outcome reward assignment relies on an Outcome Verifier that uses rule matching for tasks with predefined rules in AndroidWorld and an LLM Judge (GPT-4o) to analyze XML and action trajectories for the others. 

\paragraph{Baselines.}We employ the UI-TARS-1.5-7B~\cite{qin2025ui} base model as our starting point. We compare our approach against standard RL baselines for GUI agents including online GRPO and PPO given same RL training time budget, as well as great proprietary and open-source models. To ensure fair comparison, we re-evaluate all open-weight models. We report the mean and standard deviation across three runs.

\subsection{Main Results}
\paragraph{\ours significantly improves the model's performance, making baseline model surpass existing powerful models and methods.}
As presented in the Table~\ref{tab:main-results}, \ours yields substantial performance enhancements for the UI-TARS-1.5-7B. Specifically, it raises the SR on AndroidLab from 31.9\% to 39.4\% and AndroidWorld from 32.8\% to 41.1\%. Remarkably, our method enables the model to outperform powerful proprietary models on AndroidWorld such as Claude-Sonnet-4 with Set-of-Mark prompting which achieves 41.0\%. These results validate the effectiveness of our RL strategy. While methods like GRPO and PPO also demonstrate performance gains, our approach achieves even stronger results given an identical budget of online training time.

\paragraph{\SSMA paradigm can increase RL training efficiency compared to \SSSA methods.}
As illustrated in Figure~\ref{fig:eff_inter}, \ours conducts fewer environment interactions under the same training time due to more actions sampling and model updates per online states. However, as shown in Figure~\ref{fig:eff_sample}, the total number of sampled actions in our approach is substantially higher, which means significantly more samples for policy updates under a fixed training budget. Consequently, our method outperforms PPO by 5.2\% given the same training time, and achieves a comparable SR with $1.4\times$ higher efficiency than \sssa methods including GRPO and PPO, as illustrated in Figure~\ref{fig:eff_sr}.
These results suggest that \ours provides greater online agent RL training efficiency gains.

    
    
    

\begin{table}[!t]
\centering
\small
\resizebox{0.5\textwidth}{!}{
\begin{tabular}{lccc} 
    \toprule
    \textbf{\makecell{Single State\\$N$ Actions}} &
    \textbf{\makecell{Training\\Time}} &
    \textbf{\makecell{AndroidLab\\SR(\%)}} &
    \textbf{\makecell{AndroidWorld\\SR(\%)}} \\
    \midrule
    \emph{number of N} \\
    \quad 1 & 1.00x & 34.8\textcolor{gray}{\scriptsize\,±0.7} & 36.8\textcolor{gray}{\scriptsize\,±1.3} \\
    \quad 2 & 1.22x & 35.5\textcolor{gray}{\scriptsize\,±1.2} & 38.5\textcolor{gray}{\scriptsize\,±2.2}\\
    \quad 4 & 1.62x & \textbf{37.0}\textcolor{gray}{\scriptsize\,±0.7} & 39.1\textcolor{gray}{\scriptsize\,±0.5}\\
    \quad 8 & 2.18x & 37.0\textcolor{gray}{\scriptsize\,±0.0} & \textbf{39.4}\textcolor{gray}{\scriptsize\,±2.0} \\
    \bottomrule
\end{tabular}
}
\caption{Analysis of action rollout times. We report total training time and success rates on AndroidLab and AndroidWorld across different numbers of samples.}
\label{tab:ssna_abl}
\end{table}
\subsection{Ablation Study}
To validate key design choices in our framework, we conduct a set of ablation studies on four key components: the number of action samples, the leave-one-out advantage estimation, the process reward, and the critic initialization strategy.

\paragraph{Increasing the sample count improves performance with a sub-linear training time increase.} We ablate the number of action samples $N$ for \ours without PRM involvement. As shown in Table~\ref{tab:ssna_abl}, SR on AndroidLab increases from 34.8\% at $N=1$ to 37.0\% at $N=4$, representing a 6.3\% improvement, with AndroidWorld exhibiting a similar trend.
However, when $N$ further increases to 8, the performance gain becomes marginal. This can be attributed to the decreasing information gain from additional samples as $N$ grows, leading to a point of diminishing returns.
In terms of training time, the cost scales sub-linearly. Specifically, $N=4$ requires only $1.62\times$ the baseline time, while $N=8$ requires $2.18\times$, far below the $N\times$ cost typical of standard online methods which adopt \sssa paradigm. This highlights the efficacy of decoupling the resampling process from the environment, which enables performance improvements with sub-linear additional cost when increasing the number of samples.

\paragraph{Incorporating leave-one-out strategy into RL training leads to more stable learning and performant agents.}We assess the effectiveness of the proposed ACLOO strategy by benchmarking it against a vanilla actor-critic implementation without the baseline. To ensure a fair comparison we train all models using identical pipelines and hyperparameters. As demonstrated in Figure~\ref{fig:bar_acloo} removing the ACLOO baseline results in significant underperformance yielding a substantially lower Success Rate of 36.0\%. Further analysis of the training curves in Figure~\ref{fig:reward_acloo} reveals that the ACLOO method induces a more stable upward trend in outcome reward throughout training. This underscores the effectiveness of the leave-one-out strategy in stabilizing training and accelerating convergence.
\begin{figure}[t]
    \centering
    \includegraphics[width=0.45\textwidth]{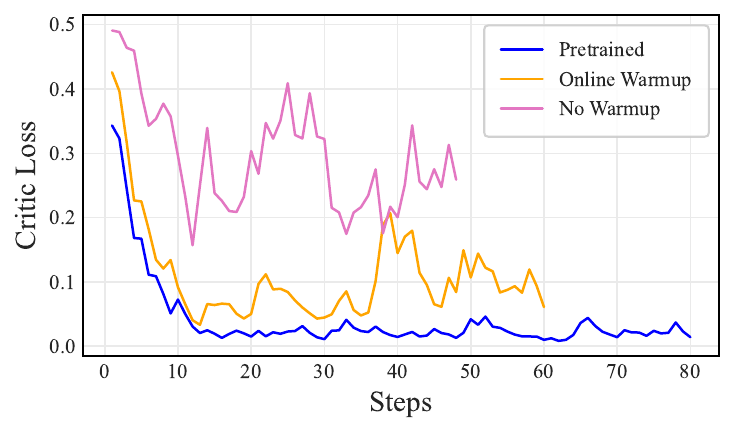}
    \vspace{-0.6em} 
    \caption{Critic loss with different initialization strategies during joint policy optimization.}
    \label{fig:loss}
\end{figure}

\paragraph{Pre-training the Q provides a stable initialization, leading to more consistent convergence of its loss during the joint optimization with the policy.}
To investigate whether critic pretraining is necessary for policy training, we conduct experiments with three initialization strategies comprising 1) training from scratch without warm-up, 2) warming up via online rollouts, and 3) pre-training on the PRM dataset. Detailed experimental configurations are provided in the Appendix~\ref{app:exp_detail}. As illustrated in the loss curves in Figure~\ref{fig:loss}, utilizing a critic without warm-up directly for policy training results in value divergence. Similarly, restricting the warm-up phase merely to the online environment induces instability during subsequent policy updates. Conversely, pre-training the value model with data aligned to the PRM enables the critic to converge stably when subsequently joint trained with the policy. This highlights the importance of having a well-pretrained critic before the RL phase.

\begin{table}[!t]
\centering
\small
\resizebox{\linewidth}{!}{%
\begin{tabular}{llll}
    \toprule
     \textbf{Usage of PRM} & \textbf{SR} & \textbf{ROR} & \textbf{Average Steps} \\
    \midrule    
    \emph{AndroidLab} \\
    \quad  wo/ PRM & 37.0 & 83.2 & 12.7 \\
    \quad  w/ PRM & 39.4 \up{2.4} & 89.8 \up{6.6} & 11.9 \down{0.8} \\
    \emph{AndroidWorld} \\
    \quad  wo/ PRM & 39.1 & 80.0 & 13.0 \\
    \quad  w/ PRM & 41.1 \up{2.0} & 85.1 \up{5.1} & 12.5 \down{0.5} \\
    \bottomrule
\end{tabular}
}
\caption{Analysis of PRM usage. We report the success rate (SR), reasonable operation ratio (ROR) of steps and average steps on AndroidLab for agent trained with and without process model.}
\label{tab:prm_abl}
\end{table}

\paragraph{Using process rewards in training boosts agent performance and leads to more reasonable intermediate steps.} We ablate the effect of the PRM as shown in Table~\ref{tab:prm_abl}. The agent trained with the PRM gets higher SR than without PRM on both AndroidLab (39.4\% \emph{vs.} 37.0\%) and AndroidWorld (41.1\% \emph{vs.} 39.1\%), confirming its quantitative benefit. We also report another metric taken from the AndroidLab benchmark, the reasonable operation ratio (ROR), which evaluates whether an action is reasonable based on the resulting changes on the screen. With PRM, ROR improves by 6.6\% and 5.1\% respectively on the two benchmarks, suggesting that incorporating PRM can enhance the reasonableness of actions taken by the agent, ultimately improving the final success rate.
\section{Conclusion}
This work introduces \ours, a \SSMA paradigm reinforcement learning framework to improve online training efficiency. 
To reduce online interactions for action value estimation, we train a reliable value model to estimate the returns of multiple sampled actions without additional emulator overhead. Specifically, we introduce a fine-grained Process Reward Model to guide the critic online training, and a low-variance group-wise advantage estimator to stabilize policy updates. Extensive experimental results demonstrate that \ours significantly enhances online training efficiency while leading to substantial improvement in performance.
\section*{Limitations}
Despite the effectiveness of \ours, there are several limitations that point to future work. First, we improve sample efficiency mainly from an algorithmic perspective, without optimizing the system architecture for large-scale parallelization; integrating our method into advanced engineering pipelines~\cite{fu2025areal} could further boost wall-clock training efficiency. Second, unlike methods such as MobileRL~\cite{xu2025mobilerl} and MAI-UI~\cite{zhou2025maiuitechnicalreportrealworld}, which apply extensive supervised fine-tuning (SFT) on human annotations before RL, we perform RL directly on a base GUI model. Due to budget constraints, we do not study SFT here, though it could, in principle, raise the performance upper bound. Finally, our reward signal depends, to some extent, on the reliability of the Outcome Verifier: occasional hallucinations in GPT-4o make outcome rewards imperfect. Future research should aim to develop more reliable verification methods for online agentic training that avoid LLM judges’ hallucinations and the labor-intensive nature of manually designed rules, while maintaining high precision.

\section*{Acknowledgments}
This work is supported by the "Leading Goose + X" Science and Technology Program of Zhejiang Province of China (2025C02104).

\bibliography{custom, llm, rl, uiagent, bench}

\appendix

\newpage
\clearpage
\section{Implementation Details}
\label{app:exp_detail}
\subsection{\ours}
\label{app:a1}
\paragraph{Policy Training.}Our policy models are trained via full-parameter fine-tuning on a single node equipped with 8 NVIDIA A100 GPUs with 80 GB memory. We implement the training pipeline using a modified version of the verl framework~\cite{sheng2024hybridflow}. To optimize efficiency, we employ a number of parallel emulator environments equal to the batch size during training. Additionally, we leverage the vLLM engine~\cite{kwon2023efficient} to accelerate inference during the rollout phase, utilize the Fully Sharded Data Parallel (FSDP2) backend for distributed training, and enable BF16 mixed-precision. The critic model is also pretrained alone using FSDP2 with data mentioned in Section~\ref{sec:app_bench_data} and then co-trained with the policy model. If the Process Reward Model is incorporated, it is deployed together with policy model with vLLM engine. During training, the policy model has access to the complete interaction history. All experiments in the main text are conducted using the same settings as the main experiment, unless otherwise specified. Detailed hyperparameters are listed in Table~\ref{tab:hyperparam}. The action spaces for GUI agent are described in Table~\ref{tab:actions}.

\paragraph{Outcome Verifier.}The Outcome Verifier functions as the mechanism for assigning outcome rewards during training~\cite{cui2025process, su2026difficultyawareagenticorchestrationqueryspecific}. For AndroidWorld tasks detailed in  Section~\ref{sec:app_bench_data}, we utilize the official built-in judge based on defined rules. For AndroidLab tasks, we adopt the DigiRL~\cite{bai2024digirl} methodology and employ GPT-4o as the Outcome Verifier. We provide the complete trajectory containing compressed XML information~\cite{xu2024androidlab} and actions for each step to enable accurate task completion judgment with prompt in Figure~\ref{fig:prompt_ov}. To validate the reliability of this method, we sample 100 trajectories evaluated by GPT-4o and enlist two graduate students who are fluent in English and majoring in STEM fields as human annotators to label them using identical criteria. Each task is compensated with \$0.5. The agreement rate between human annotations and the rewards assigned by the Outcome Verifier is 92\%, demonstrating its reliability.

\paragraph{PRM Training.}Our Process Reward Model (PRM) is trained via full-parameter Supervised Fine-Tuning (SFT) using verl on a single node with 8 NVIDIA A100 GPUs with 80 GB memory. We employ the FSDP2 backend with a batch size of 32 for 2 epochs. We optimize the model using AdamW~\cite{adamw2017} with $\beta_1=0.9$, $\beta_2=0.95$, and a weight decay of $0.01$. The learning rate is scheduled with a cosine decay strategy, peaking at 1e-5 after a warmup phase covering 10\% of the total training steps. To enhance stability, we apply gradient clipping with a maximum norm of $1.0$.

\subsection{Reproduction}
\paragraph{PPO and GRPO.}We reproduce the PPO and GRPO algorithms following the identical hardware configuration described in Appendix~\ref{app:a1} which utilizes a single node equipped with eight 80GB NVIDIA A100 GPUs. We employ the native implementations provided by the verl framework~\cite{sheng2024hybridflow} and utilize the same Outcome Verifier for training reward assignment. For PPO, the value model is also pretrained using the same dataset with \ours using Generalized Advantage Estimation~\cite{john2016gae} with $\lambda$ set to 1.0 which functions equivalently to the Monte Carlo return. We use a constant learning rate of 1e-6 for the actor and 1e-5 for the value model respectively without applying a KL coefficient. For GRPO we set the group size to 4, with learning rate of 1e-6 for the actor. To ensure fair comparison we maintain a consistent training time budget across these experiments.  All third-party artifacts used in this work (VERL, AndroidWorld, AndroidLab) are released under the Apache License 2.0. The full license texts are available in their respective official repositories.

\paragraph{Evaluation of other models.}Since not all existing baselines report performance on both AndroidLab and AndroidWorld benchmarks or provide the necessary granular data, we conduct a re-evaluation of selected open-weight GUI agent models. For Qwen2.5VL-32B~\cite{bai2025qwen2}, UI-TARS-72B-DPO~\cite{qin2025ui}, Qwen2.5-VL-7B-Instruct~\cite{bai2025qwen2}, and UI-TARS-1.5-7B~\cite{qin2025ui}, we adopt the identical input format utilized by \ours. Conversely, for OS-Genesis-7B-AW~\cite{sun2025genesis} and AgentCPM-GUI-8B~\cite{zhang-etal-2025-agentcpm}, we adhere to their officially recommended input formats. Furthermore, we ensure the output action space consistent with the official specifications provided by the model publishers. Besides, due to budget constraints, the results of the proprietary models are taken from UI-S1~\cite{lu2025uis1} and MobileRL~\cite{xu2025mobilerl}, which report performance of gemini-1.5-pro, GPT-4o and Claude-Sonnet-4 under the Set of Marks (SoM) strategy.  

\section{Benchmarks and Data}
\label{sec:app_bench_data}
\subsection{Environment}
Following prior work~\cite{xu2025mobilerl,xiao2025ui,chen2025gui}, we construct our environment for training from AndroidLab~\cite{xu2024androidlab} and AndroidWorld~\cite{rawles2024androidworld}. 
All experiments are conducted in a controlled emulator environment with a pre-configured Android 13 system at API Level 33 equipped with the complete Google Mobile Services suite. The agent interacts with emulators by reasoning~\cite{wu2026atlas, ma2026thinkingblueprintsassistingvisionlanguage, MaChenZhangWuDing2025} and generating function-call-like commands, which are subsequently executed via the \textit{uiautomator2} tool. The emulators include applications for commonly used tasks such as bookkeeping, navigation, and calendar management. These tasks cover both execution and querying scenarios. After RL, the agent's performance is evaluated using strict, rule-defined matching criteria~\cite{ding2026octobenchbenchmarkingscaffoldawareinstruction, lin2026mmdocr1trainingagentslong}.

\subsection{Benchmarks}
\paragraph{AndroidLab.}
AndroidLab serves as an online benchmark platform designed to evaluate autonomous GUI agents within the Android environment. It comprises 138 tasks spanning nine mobile applications including Zoom, Pi Music Player, Bluecoins and so on. Different from AndroidWorld, there is no randomness in evaluation tasks and initialization scenarios, which means each task has a fixed initial state and expected outcome. Performance metrics include Success Rate (SR) and Reasonable Operation Ratio (ROR).

\paragraph{AndroidWorld.}
AndroidWorld functions as a comprehensive online benchmark for autonomous GUI agents featuring 116 tasks across 20 distinct applications. Task categories encompass audio recording, content editing, gaming, and scheduling. To ensure scenario diversity, these tasks are dynamically generated using variable input parameters and adaptive initialization states. We utilize Success Rate (SR) as the primary metric to evaluate agent performance. We also additionally use the Reasonable Operation Ratio (ROR) metric.

\subsection{Data}
\label{sec:data}
\paragraph{AndroidControl.}
AndroidControl represents a large-scale dataset consisting of 15,283 demonstrations of everyday tasks involving 833 Android applications. We utilize AndroidControl exclusively to construct our Process Reward training dataset, which is detailed in subsequent paragraphs.

\paragraph{Policy Training Tasks.}The reinforcement learning phase requires only unsupervised task instructions. We construct our task pool by leveraging the AndroidLab and AndroidWorld environments. Specifically, we automatically generate candidate tasks based on accessible applications in AndroidLab and synthesize additional tasks using randomized parameters within AndroidWorld. Following a manual verification process to ensure feasibility and strictly exclude overlaps with the test set, we compile a final dataset of 2,000 training tasks.
\paragraph{Process Reward Model Training Dataset.}We compile our dataset from AndroidControl and pre-collected online trajectories, filtering for successful and non-redundant sequences via GPT-4o and manual verification. For each state, we generate eight candidate reasoning-action pairs using UI-TARS-1.5-7B. Candidates are labeled positive if their actions align with the ground truth while mismatches are labeled negative. Following GUIOdyssey~\cite{lu2025guiodyssey}, we consider coordinates correct if they fall within a distance of 14\% of the screen width from the ground truth. We further synthesize high-quality reasoning components for process judgment using GPT-4o. The final dataset comprises 20k samples with a balanced 1:1 ratio between positive and negative examples, as shown in Figure~\ref{fig:prm_data}. Subsequently, we adapt these data points to pre-train the Q function by mapping the label True/False to 1/0 value score.
\begin{figure}[ht]
    \centering
    \vspace{-2em}
    \includegraphics[width=0.48\textwidth]{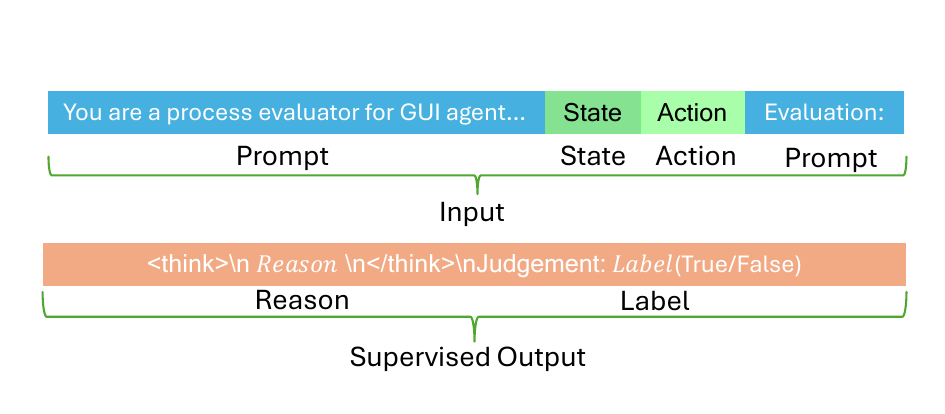}
    \vspace{-2em}
    \caption{Process Reward Model training data format.}
    \label{fig:prm_data}
\end{figure}
\paragraph{License.}All third-party artifacts used in this work are released under the Apache License 2.0, and our use is consistent with their standard open-source terms for academic research.
\begin{table*}[htbp]
\centering
\begin{tabular}{lll}
\toprule
Component & Hyperparameter & Value \\
\midrule
Data & Max Prompt Length & 32768 \\
Data & Max Response Length & 512 \\
Data & Train Batch Size & 8 \\
Actor/Policy & Strategy (Parallelism) & FSDP2 \\
Actor/Policy & PPO Micro Batch Size/GPU & 1 \\
Actor/Policy & Learning Rate (LR) & 1e-6 \\
Actor/Policy & Gradient Clipping & 1.0 \\
Actor/Policy & Clip Ratio & 0.2 \\
Rollout \& Sampling & Sampling Temperature & 1.0 \\
Rollout \& Sampling & Max New Tokens & 512 \\
Rollout \& Sampling & Max Turns & 25 \\
Rollout \& Sampling & Max Pixels & 1270180 \\
Rollout \& Sampling & Min Pixels & 256 \\
Reward & Process Reward Weight \textbf{$\omega_p$} & 0.2 \\
Reward & Outcome Reward Weight \textbf{$\omega_o$} & 1 \\
Reward & Discount factor \textbf{$\lambda$} & 0.95 \\
Critic & Learning Rate (LR) & 1e-5 \\
Critic & Clip Range Value & 0.5 \\
Critic & Warmup Ratio & 0.1 \\
\bottomrule
\end{tabular}
\caption{Main hyperparameters in \ours.}
\label{tab:hyperparam}
\end{table*}
\begin{table*}[htbp]
\centering
\begin{tabular}{ll}
\toprule
\textbf{Action} & \textbf{Definition} \\
\midrule
Click(x, y) & Clicks at coordinates (x, y). \\
Scroll(x1, y1, x2, y2) & Scrolls from (x1, y1) to (x2, y2). \\
Drag(x1, y1, x2, y2) & Drags from (x1, y1) to (x2, y2). \\
Type(content) & Types the specified content. \\
Wait() & Pauses for a brief moment. \\
Finished(content) & Marks the task as complete. \\
LongPress(x, y) & Long presses at (x, y). \\
PressBack() & Presses the "back" button. \\
PressHome() & Presses the "home" button. \\
PressEnter() & Presses the "enter" key. \\
\bottomrule
\end{tabular}
\caption{Operation action space for GUI agent in \ours.}
\label{tab:actions}
\end{table*}

\newpage 
\onecolumn
\begin{figure}[htbp] 
    \centering 
    \begin{tcolorbox}[
        title=Prompt for \ours,
        colback=blue!5!white,
        colframe=blue!75!black,
        fonttitle=\bfseries,
        arc=5pt,
        boxrule=1pt,
        colbacktitle=blue!50!white,
    ]
    \textbf{Task Description}
    You are a GUI agent. You are given a task and your action history, with screenshots. You need to perform the next action to complete the task.

    \textbf{Output Format}

    Thought: ...\\
    Action: ...

    \textbf{Action Space}

    click(start\_box=`<|box\_start|>(x1,y1)<|box\_end|>')\\
    long\_press(start\_box=`<|box\_start|>(x1,y1)<|box\_end|>')\\
    type(content=`xxx')\\
    scroll(start\_box=`<|box\_start|>(x1,y1)<|box\_end|>', end\_box=`<|box\_start|>(x2,y2)<|box\_end|>')\\
    open\_app(app\_name=`')\\
    press\_home()\\
    press\_back()\\
    finished(content=`') \# Submit the task regardless of whether it succeeds or fails.

    \textbf{Note}\\
    - Use English in Thought part.\\
    - First summarize your previous actions, then write a small plan and finally summarize your next action (with its save target element) in one sentence in Thought part.\\
    Mobile and UI Agent Interaction History: \{interaction\_history\}
    \end{tcolorbox}
    \caption{Prompt for \ours. This prompt is consistent with the official prompt provided by UI-TARS-1.5-7B.} 
    \label{fig:prompt_uiagent} 
\end{figure}
\begin{figure}[H] 
    \centering 
    \begin{tcolorbox}[
        title=Prompt for Outcome Verifier,
        colback=blue!5!white,
        colframe=blue!75!black,
        fonttitle=\bfseries,
        arc=5pt,
        boxrule=1pt,
        colbacktitle=blue!50!white,
    ]
    \textbf{Task Overview:}\\
    You are an expert evaluator for determining the success of GUI tasks. You will be provided with the following information:\\
    1. The task description.\\
    2. Mobile and UI Agent Interaction History including the step-by-step page state in compressed XML format and the agent's action for each step.

    \textbf{Scoring Rule:}\\
    You need to judge if the UI Agent completed the task based on the interaction trajectory. You should return True or False according to your judgment.

    \textbf{Output Format:}\\
    <analysis> [Your analysis] </analysis>\\ 
    </ans> [Your judgment] </ans>\\
    \textbf{Current Task Information:}\\
    Task Description: \{instruction\} \\
    Mobile and UI Agent Interaction History: \{interaction\_history\}
    \end{tcolorbox}
    \caption{Prompt for outcome verifier.} 
    \label{fig:prompt_ov} 
\end{figure}
\vspace{-40em}
\begin{figure}[htbp] 
    \centering 
    \begin{tcolorbox}[
        title=Prompt for Process Reward Model,
        colback=blue!5!white,
        colframe=blue!75!black,
        fonttitle=\bfseries,
        arc=5pt,
        boxrule=1pt,
        colbacktitle=blue!50!white
    ]
    \textbf{Role Definition:} 
    You are a meticulous evaluator for an Android GUI automation agent. Your primary mission is to analyze the agent's reasoning and proposed action in the context of a given task and the current user interface. You must determine if the agent's action is a correct and logical step towards completing the task, and judge whether the operation conforms to Android system specifications.
    
    \textbf{Input Data:} 
    You will be provided with:
    1. Instruction: The high-level goal.
    2. Screenshot: A visual representation of the current GUI state.
    3. Agent's Thought and Action: The reasoning process and the specific \texttt{Action} intended.
    
    \textbf{Evaluation Criteria:}
    Your output format should be \texttt{<think>...thought process...</think> judgment:True or False}.\\
    \textbf{Return \texttt{True} if:} The proposed \texttt{Action} is logical, relevant, and productive based on a correct interpretation of the \texttt{Screenshot}.\\
    \textbf{Return \texttt{False} if:} The action is incorrect (illogical, misinterpretation of UI, redundant, or counter-productive).
    
    \noindent\rule{\textwidth}{0.4pt} 
    
    \begin{minipage}[t]{0.70\textwidth} 
        \textbf{Example 1 (Correct Action):}\\
        \textbf{Instruction:} Record an audio clip and save it with name ``F3tb\_presentation.m4a'' using the Audio Recorder app.\\
        \textbf{Screenshot:} [Home screen with ``Audio Recorder'' icon visible.]\\
        \textbf{Agent's Thought and Action:}
        Thought: I need to open the app. I see the  ``Audio Recorder'' icon. I will tap it.
        Action: \texttt{click(start\_box=`(635,520)')}\\
        \textbf{Evaluation:}
        <think>
        Agent's Logic Analysis: The agent correctly identifies the first step and locates the icon. 
        Action Validation: The action is the most direct and logical step.
        </think>
        judgment:True
    \end{minipage}%
    \hfill 
    \begin{minipage}[t]{0.28\textwidth} 
        \vspace{0pt} 
        \centering
        \includegraphics[width=0.5\linewidth]{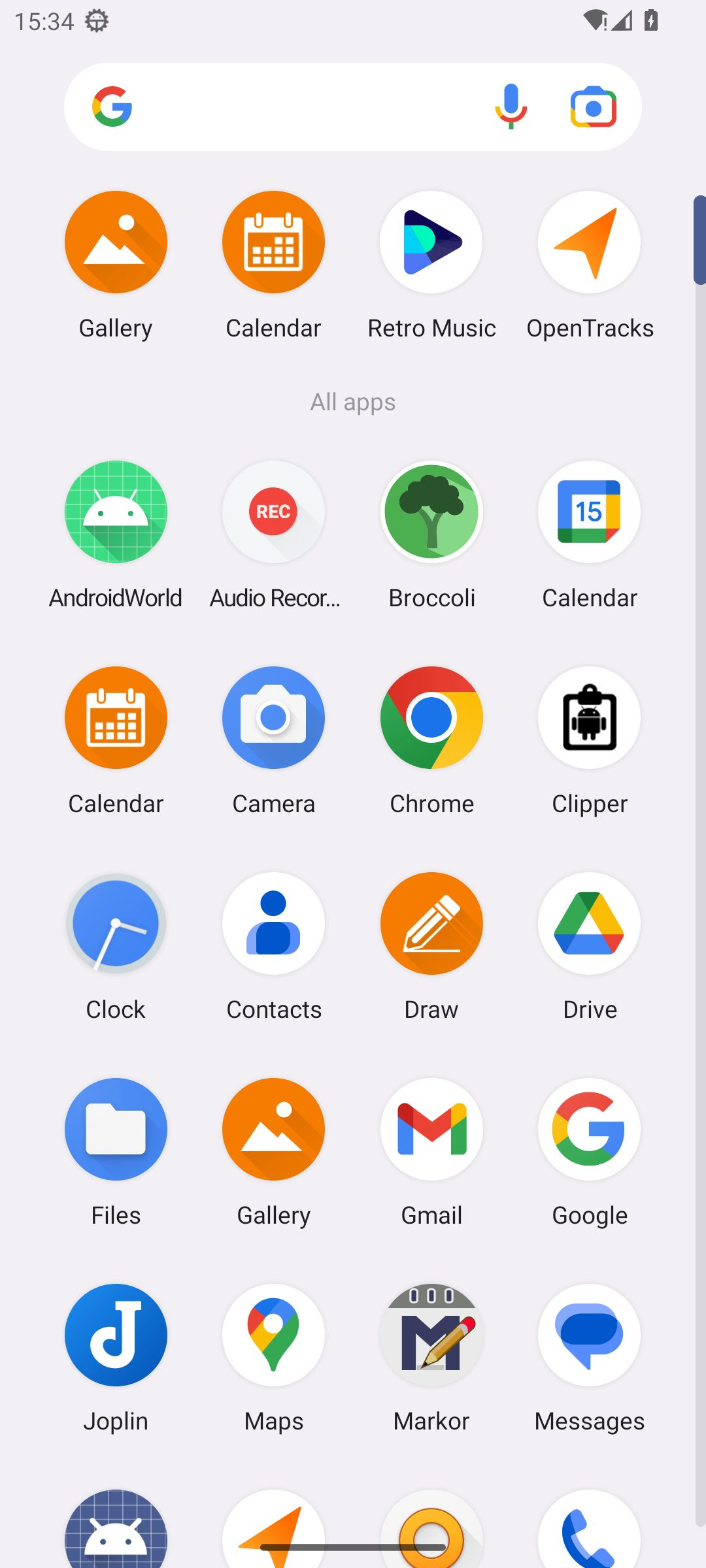}
        \captionof*{figure}{\footnotesize Example 1 Screenshot} 
    \end{minipage}
    
    
    \begin{minipage}[t]{0.70\textwidth} 
        \textbf{Example 2 (Incorrect Action):}\\
        \textbf{Instruction:} Record an audio clip and save it with name ``F3tb\_presentation.m4a'' using the Audio Recorder app.\\
        \textbf{Screenshot:} [Inside Audio Recorder app. ``Get Start'' button visible.]\\
        \textbf{Agent's Thought and Action:}
        Thought: I don't see the app icon, so I must be in the wrong place. I need to go back.
        Action: \texttt{press\_back()}\\
        \textbf{Evaluation:}
        <think>
        Agent's Logic Analysis: The agent used flawed reasoning; it failed to recognize it is already in the app.
        Action Validation: \texttt{press\_back()} exits the app, which is counter-productive. Correct action is clicking get start.
        </think>
        judgment:False
    \end{minipage}%
    \hfill
    \begin{minipage}[t]{0.28\textwidth} 
        \vspace{0pt}
        \centering
        \includegraphics[width=0.5\linewidth]{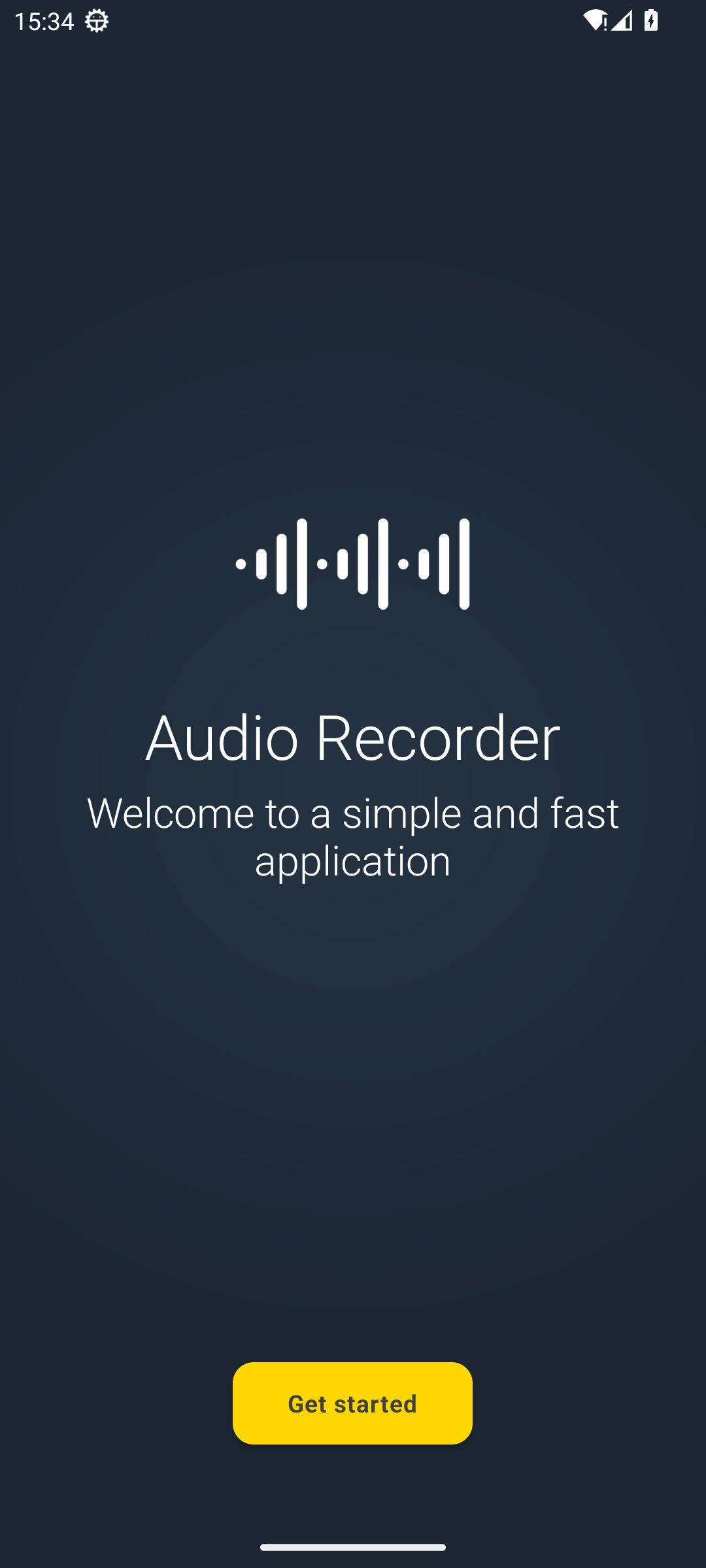} 
        \captionof*{figure}{\footnotesize Example 2 Screenshot}
    \end{minipage}
    
    \noindent\rule{\textwidth}{0.4pt} 
    
    \textbf{Now, evaluate the following scenario:}
    
    \textbf{Instruction:}
    \{instruction\}
    
    \textbf{Agent's Thought and Action:}
    \{agent\_thought\_and\_action\}
    
    \textbf{Evaluation:}
    \end{tcolorbox}
    \caption{Prompt for Process Reward Model.} 
    \label{fig:prompt_prm} 
\end{figure}
\begin{figure}[htbp] 
    \centering 
    \begin{tcolorbox}[
        title=Prompt for GPT-4o to Generate Reason for PRM Training,
        colback=blue!5!white,
        colframe=blue!75!black,
        fonttitle=\bfseries,
        arc=5pt,
        boxrule=1pt,
        colbacktitle=blue!50!white,
    ]
    \textbf{Task Overview:}\\
    You are an expert evaluator for an Android GUI Agent. Your mission is to analyze the agent's behavior at each step and determine if its intended action is correct and logical for accomplishing a given task.
    
    \textbf{1. Context}\\
    The agent interacts with an android operating system. It receives a high-level instruction from a user and attempts to complete it by performing a sequence of actions on the device's GUI. The process is sequential: observe screen $\rightarrow$ generate thought/action $\rightarrow$ execute action $\rightarrow$ repeat.
    The agent's interaction is restricted to predefined actions: \\
    click(start\_box=`[x1, y1, x2, y2]'), long\_press(start\_box=`<|box\_start|>(x1,y1)<|box\_end|>'), scroll(start\_box=`<|box\_start|>(x1,y1)<|box\_end|>', end\_box=`<|box\_start|>(x2,y2)<|box\_end|>'), drag(start\_box=`<|box\_start|>(x1,y1)<|box\_end|>', type(content=`xxx'),  open\_app(app\_name=`'), press\_back(), press\_home(), finished(content=`xxx')
    
    \textbf{2. Your Task}\\
    You will be provided with four pieces of information for a single step: instruction, screenshot, historical actions, ground truth action, agent's thought and action and ground truth judgment. Your task is to critically evaluate the agent's step using an analytical process: 1) Understand the instruction. 2) Analyze the History and Screenshot. 3) Evaluate the agent's action given the ground truth judgment. Output \textbf{1 (Correct/Reasonable Action)} if: Direct progress, logical sub-task, handling obstacles, sensible exploration, appropriate action type, error correction. Output \textbf{0 (Incorrect/Unreasonable Action)} if: Contradicts instruction, illogical action, flawed reasoning, no progress, inefficient action, wrong action type, perpetuating errors.
    
    \textbf{3. Output Format}\\
    Your response MUST be a valid JSON object with two keys:\\
    1. "result": Integer (1 for Correct, 0 for Incorrect)\\
    2. "reason": Concise string explaining your judgment
    
    \textbf{4. Examples}\vspace{-0.5em}
    \begin{tcolorbox}[
        colback=white, 
        colframe=white, 
        fonttitle=\bfseries\footnotesize,
        fontupper=\scriptsize, 
        boxrule=0.5pt,
        left=2pt, right=2pt, top=2pt, bottom=2pt,
        arc=2pt
    ]
    \textbf{Example 1:} Correctly Handling an Obstacle (Correct)\\
    \textbf{Instruction:} Open a new private tab in the browser.\\
    \textbf{Agent's Thought and Action:} Thought: A system dialog has appeared asking to set the default browser. This is not required for my current task. I will dismiss this dialog by clicking `Cancel' to continue. Action: click(start\_box=`(700, 1300)')\\
    \textbf{Expected Output:} \{``result": 1, ``reason": ``The agent correctly identified a system dialog as an obstacle not relevant to the main task. Its decision to dismiss the dialog is a logical and efficient way to handle the interruption and proceed with the instruction."\}
    \end{tcolorbox}
    \begin{tcolorbox}[
        colback=white, 
        colframe=white, 
        fonttitle=\bfseries\footnotesize,
        fontupper=\scriptsize, 
        boxrule=0.5pt,
        left=2pt, right=2pt, top=2pt, bottom=2pt,
        arc=2pt
    ]
    \textbf{Example 2:} Error Based on Historical Context (Incorrect)\\
    \textbf{Instruction:} Delete the contact ``John Smith".\\
    \textbf{Agent's Thought and Action:} Thought: A confirmation pop-up has appeared to delete the contact. I need to confirm the deletion to complete the task. Action: click(start\_box=`(600, 1200)')\\
    \textbf{Expected Output:} \{``result": 0, ``reason": ``The action perpetuates a previous error. The agent is about to delete `John Appleseed', but the instruction was to delete `John Smith'. This stems from an incorrect selection in a previous step, and proceeding would fail the task."\}
    \end{tcolorbox}
    \textbf{Scoring Rule:}\\
    You need to judge if the agent's current action is correct and logical based on the instruction, screenshot, historical context, and the agent's reasoning.\\
    \textbf{Output Format:}\\
    <analysis> [Your analysis] </analysis>\\ 
    </ans> [Your judgment] </ans>\\
    
    \textbf{Your Turn:}\\
    Instruction:\{instruction\}. Historical Actions:\{history\}. Ground truth action:\{ground\_truth\}. Agent’s Thought and Action:\{agent\_thought\_and\_action\}. Ground Truth judgment:\{judgment\}.\\
    \textbf{Your Output:}\\
    \end{tcolorbox}
    \caption{Prompt for GPT-4o to generate reason for PRM Training.} 
    \label{fig:prompt_prm_cot} 
\end{figure}
\newpage
\clearpage
\begin{onecolumn}
\section{Pseudocode}
\label{sec:algorithm}
\begin{algorithm}
\caption{Android Coach Framework}
\label{alg:android_coach}
\begin{algorithmic}[1]
    \State \textbf{Initialize:} Initial actor parameters $\theta$, initial critic parameters $\phi$.
    \State \textbf{Given:} Instruction Pool $\mathcal{I}$, process reward model(PRM), outcome reward verifier(OV).

    \For{each iteration}
        \State \textcolor{green!50!black}{\textbf{\# Phase 1: Actor Data Collection}} 
        \For{each Android step $t$}
            \State Given $\mathcal{I}$, execute action $a_t \sim \pi_\theta(\cdot | s_t)$ and store finished trajectory $\tau$.\Comment\textcolor{green!50!black}{Online Interaction}
        \EndFor
        
        \State \textcolor{green!50!black}{\textbf{\# Phase 2: Assign Returns}} 
        \For{each trajectory $\tau$}
            \State $R_{\text{outcome}} \gets \text{OV}(\tau, \text{Instruction})$
            \For{step $t \gets T$ \textbf{downto} $1$ in trajectory $\tau$}
                \State $r_t^p \gets \text{PRM}(a_t, s_t)$ 
                \State $R_t \gets \text{MC estimation}\left( r_p^{t:T}, r_{\text{O}} \right)$ 
                \State Add $(s_t, a_t, R_t)$ to replay buffer $\mathcal{D}$.
            \EndFor
            
        \EndFor

        \State \textcolor{green!50!black}{\textbf{\# Phase 3: Update Critic}} 
        \For{each critic step}
            \State Sample batch of $(s_t, a_t, R_t) \sim \mathcal{D}$.
            \State Update $\phi$ by clipped MSE loss in Equation~\ref{eq:critic_mse}.
        \EndFor

        \State \textcolor{green!50!black}{\textbf{\# Phase 4: Update Actor}} 
        \For{each actor step}
            \State Sample batch of states $\{s\} \sim \mathcal{D}$
            \State Generate $K$ responses: $\{a_1, \dots, a_K\} \sim \pi_\theta(\cdot|s)$ \Comment\textcolor{green!50!black}{Single State Multiple Actions} 
            \State Compute advantages $\hat{A}(s, a_i)$ by ACLOO: 

            \State \hspace{\algorithmicindent} $Q_i \gets Q_\phi(s, a_i)$ 
            \State \hspace{\algorithmicindent} $\hat{A}(s, a_i) \gets Q_i - \frac{1}{k-1}\sum_{j \neq i} Q_j$
            \State Update $\theta$ by PPO loss in Equation~\ref{eq:ppo_clip_aligned}.
        \EndFor
    \EndFor
\end{algorithmic}
\end{algorithm}
\end{onecolumn}
\newpage
\newtheorem{theorem}{Theorem}
\newtheorem{lemma}{Lemma}
\theoremstyle{definition}
\newtheorem{definition}{Definition}
\theoremstyle{remark}

\section{Lemma}
\label{sec:lemma}
Let $s_t$ be the current state. We sample $k$ independent and identically distributed (i.i.d.) actions from our policy $\pi_{\theta}(\cdot | s_t)$:
$$a^{(1)}, a^{(2)}, \dots, a^{(k)} \sim \pi_{\theta}(\cdot | s_t)$$

For each $i$-th sample $a^{(i)}$, we compute its Q-value $Q(s_t, a^{(i)})$ and define a leave-one-out baseline $b_i$:
$$b_i = \frac{1}{k-1} \sum_{j \neq i} Q(s_t, a^{(j)})$$

The advantage for the $i$-th sample is:
$$\hat{A}^{(i)} = Q(s_t, a^{(i)}) - b_i$$

While the PPO $L^{CLIP}$ objective is inherently biased for stabilization, we justify our choice of the ACLOO estimator by proving that it is statistically sound. Specifically, it is a low-variance estimator that does not introduce any bias when applied to the standard policy gradient theorem. This demonstrates its validity as a high-quality advantage signal. The policy gradient estimator for this sample is:
$$g_i = \hat{A}^{(i)} \nabla_{\theta} \log \pi_{\theta}(a^{(i)} | s_t)$$
 We will now prove that this estimator $g_i$ is \textbf{unbiased} and has \textbf{reduced variance}.

\subsection{Proof of Unbiasedness}

\begin{lemma}[Unbiased Estimator]
The policy gradient estimator $g_i$ is an unbiased estimator of the true policy gradient $\nabla_{\theta} J(\theta)$.
\end{lemma}

\begin{proof}
We prove that the expected value of our estimator $g_i$ is equal to the true policy gradient $\nabla_{\theta} J(\theta)$.

The true policy gradient is defined as:
$$\nabla_{\theta} J(\theta) = \E_{a \sim \pi_{\theta}} [ Q(s_t, a) \nabla_{\theta} \log \pi_{\theta}(a | s_t) ]$$

The expectation of our estimator $g_i$ is taken over all $k$ i.i.d. samples:
$$\E[g_i] = \E \left[ \left( Q(s_t, a^{(i)}) - b_i \right) \cdot \nabla_{\theta} \log \pi_{\theta}(a^{(i)} | s_t) \right]$$

By linearity of expectation, we split this into two terms:
$$\E[g_i] = \E[Q(s_t, a^{(i)}) \nabla_{\theta} \log \pi_{\theta}(a^{(i)} | s_t)] - \E[b_i \cdot \nabla_{\theta} \log \pi_{\theta}(a^{(i)} | s_t)]$$

\textbf{1. Analyzing the first term:}
Since $a^{(i)}$ is a sample drawn from $\pi_{\theta}(\cdot | s_t)$, the first term is, by definition, the true policy gradient:
$$\E[Q(s_t, a^{(i)}) \nabla_{\theta} \log \pi_{\theta}(a^{(i)} | s_t)] = \nabla_{\theta} J(\theta)$$

\textbf{2. Analyzing the second term (the bias term $B$):}
$$B = \E[b_i \cdot \nabla_{\theta} \log \pi_{\theta}(a^{(i)} | s_t)]$$

The key insight is that our $k$ samples are i.i.d.
\begin{itemize}
    \item The baseline $b_i = \frac{1}{k-1} \sum_{j \neq i} Q(s_t, a^{(j)})$ is a random variable that depends only on the samples $\{a^{(j)}\}_{j \neq i}$.
    \item The gradient term $\nabla_{\theta} \log \pi_{\theta}(a^{(i)} | s_t)$ is a random variable that depends only on the sample $a^{(i)}$.
\end{itemize}
Because $a^{(i)}$ is \textbf{statistically independent} of $\{a^{(j)}\}_{j \neq i}$, the random variables $b_i$ and $\nabla_{\theta} \log \pi_{\theta}(a^{(i)} | s_t)$ are also \textbf{statistically independent}.

For independent random variables $X$ and $Y$, $\E[X Y] = \E[X] \E[Y]$. Therefore:
$$B = \E[b_i] \cdot \E[\nabla_{\theta} \log \pi_{\theta}(a^{(i)} | s_t)]$$

We now compute the expectation of the gradient term:
\begin{align*}
\E[\nabla_{\theta} \log \pi_{\theta}(a^{(i)} | s_t)] &= \sum_a \pi_{\theta}(a | s_t) \cdot \nabla_{\theta} \log \pi_{\theta}(a | s_t) \\
&= \sum_a \pi_{\theta}(a | s_t) \cdot \frac{\nabla_{\theta} \pi_{\theta}(a | s_t)}{\pi_{\theta}(a | s_t)} \\
&= \sum_a \nabla_{\theta} \pi_{\theta}(a | s_t) \\
&= \nabla_{\theta} \left( \sum_a \pi_{\theta}(a | s_t) \right) \\
&= \nabla_{\theta} (1) = 0
\end{align*}
Substituting this result back into the bias term $B$:
$$B = \E[b_i] \cdot 0 = 0$$

\textbf{3. Conclusion:}
The bias term is zero. Thus, the expectation of our estimator is the true policy gradient:
$$\E[g_i] = \nabla_{\theta} J(\theta) - 0 = \nabla_{\theta} J(\theta)$$
This proves the estimator is \textbf{unbiased}.
\end{proof}

\subsection{Proof of Variance Reduction (via Shift-Invariance)}

\begin{lemma}[Variance Reduction]
The advantage estimator $\hat{A}^{(i)}$ is invariant to an arbitrary constant shift $C$ added to the Q-function, which centers the advantage estimates and reduces variance.
\end{lemma}

\begin{proof}
We prove that $\hat{A}^{(i)}$ is shift-invariant. Let $Q'(s, a) = Q(s, a) + C$ be the shifted Q-function for any constant $C \in \R$.

The new advantage $\hat{A}'^{(i)}$ is:
$$\hat{A}'^{(i)} = Q'(s_t, a^{(i)}) - b'_i$$

First, we compute the new baseline $b'_i$ using the shifted $Q'$-values:
\begin{align*}
b'_i &= \frac{1}{k-1} \sum_{j \neq i} Q'(s_t, a^{(j)}) \\
&= \frac{1}{k-1} \sum_{j \neq i} \left( Q(s_t, a^{(j)}) + C \right) \\
&= \left( \frac{1}{k-1} \sum_{j \neq i} Q(s_t, a^{(j)}) \right) + \left( \frac{1}{k-1} \sum_{j \neq i} C \right) \\
&= b_i + \frac{1}{k-1} \left( (k-1) \cdot C \right) \\
&= b_i + C
\end{align*}
Now, substitute $Q'$ and $b'_i$ back into the expression for $\hat{A}'^{(i)}$:
\begin{align*}
\hat{A}'^{(i)} &= Q'(s_t, a^{(i)}) - b'_i \\
&= \left( Q(s_t, a^{(i)}) + C \right) - \left( b_i + C \right) \\
&= Q(s_t, a^{(i)}) - b_i \\
&= \hat{A}^{(i)}
\end{align*}
\textbf{Conclusion:}
Since $\hat{A}'^{(i)} = \hat{A}^{(i)}$, the advantage estimator is invariant to any constant shift $C$. This demonstrates that $\hat{A}^{(i)}$ measures the relative quality of $a^{(i)}$ compared to the average of its peers, effectively centering the advantage values. This centering property dramatically \textbf{reduces the variance} of the gradient estimator $g_i$.
\end{proof}

\end{document}